\documentclass[11pt]{article}
\usepackage{enumerate}
\usepackage{pdfsync}
\usepackage[OT1]{fontenc}

\usepackage{url}            
\usepackage{booktabs}       
\usepackage{amsfonts}       
\usepackage{nicefrac}       
\usepackage{microtype}      
\usepackage{smile}
\usepackage{color}

\usepackage[colorlinks,
linkcolor=red,
anchorcolor=blue,
citecolor=blue
]{hyperref}

\usepackage{fullpage}
\usepackage{setspace}
\usepackage{tabularx}
\usepackage{float}
\usepackage{wrapfig,lipsum}
\usepackage{enumitem}
\usepackage{natbib}
\usepackage{multirow}

\makeatletter
\newcommand*{\rom}[1]{\expandafter\@slowromancap\romannumeral #1@}
\makeatother
\usepackage{pgfplots}
\usepackage{smile}
\usepackage{algorithmic}
\usepackage{subfigure}
\usepackage{enumitem}

\def \la {\langle}
\def \ra {\rangle}
\def \gradsb {\nabla_{\bs}}
\def \gradgam {\nabla_{\bgamma}}
\def \hessb {\nabla_{\bs}^2}
\def \hesgam {\nabla_{\bgamma}^2}
\def \gammamax{\gamma_{\max}}
\def \gammamin{\gamma_{\min}}
\def \smax {s_{\max}}
\def \dd {\text{d}}
\newtheorem{condition}[theorem]{Condition}
\allowdisplaybreaks

\title{\huge Rank Aggregation via Heterogeneous Thurstone Preference Models}

\author{
	Tao Jin\thanks{Department of Computer Science, University of Virginia, Charlottesville, VA 22904; e-mail: {\tt taoj@virginia.edu}}~\footnotemark[3]
	~~~and~~~
	Pan Xu\thanks{Department of Computer Science, University of California, Los Angeles, Los Angeles, CA 90095; e-mail: {\tt panxu@cs.ucla.edu}}~\thanks{Equal contribution}
	~~~and~~~
	Quanquan Gu\thanks{Department of Computer Science, University of California, Los Angeles, Los Angeles, CA 90095; e-mail: {\tt qgu@cs.ucla.edu}}~\footnotemark[6]
	~~~and~~~
	Farzad Farnoud\thanks{Department of Electrical and Computer Engineering, University of Virginia, Charlottesville, VA 22904; e-mail: {\tt farzad@virginia.edu}}~\thanks{Co-corresponding authors.}
}
\begin{document}

\date{}
\maketitle

\begin{abstract}
We propose the Heterogeneous Thurstone Model (HTM) for aggregating ranked data, which can take the accuracy levels of different users into account. By allowing different noise distributions, the proposed HTM model maintains the generality of Thurstone's original framework, and as such, also extends the Bradley-Terry-Luce (BTL) model for pairwise comparisons to heterogeneous populations of users. Under this framework, we also propose a rank aggregation algorithm based on alternating gradient descent to estimate the underlying item scores and accuracy levels of different users simultaneously from noisy pairwise comparisons. We theoretically prove that the proposed algorithm converges linearly up to a statistical error which matches that of the state-of-the-art method for the single-user BTL model. We evaluate the  proposed HTM model and algorithm on both synthetic and real data, demonstrating that it outperforms existing methods.


\end{abstract}

\section{Introduction}

Rank aggregation refers to the task of recovering the order of a set of objects given pairwise comparisons, partial rankings, or full rankings obtained from a set of users or experts. Compared to rating items, comparison is a more natural task for humans which can provide more consistent results, in part because it does not rely on arbitrary scales. Furthermore, ranked data can be obtained not only by explicitly querying users, but also through passive data collection, i.e., by observing user behavior, for example product purchases, clicks on search engine results, choice of movies in streaming services, etc. As a result, rank aggregation has a wide range of applications, from classical social choice applications~\citep{deborda1781} to information retrieval~\citep{dwork2001a}, recommendation systems~\citep{baltrunas2010}, and bioinformatics~\citep{aerts2006a,kim2015a}.

In aggregating rankings, the raw data is often noisy and inconsistent. One approach to arrive at a single ranking is to assume a generative model for the data whose parameters include a true score for each of the items. In particular, Thurstone's preference model \citep{thurstone1927} assumes that comparisons or partial rankings result from comparing versions of the true scores corrupted by additive noise. Special cases of Thurstone's model include the popular Bradley-Terry-Luce (BTL) model for pairwise comparisons and the Placket-Luce (PL) model for partial rankings. In these settings, estimating the true scores from data will allow us to identify the true ranking of the items. Various estimation and aggregation algorithms have been developed for Thurstone's preference model and its special cases, including \citep{hunter2004a,guiver2009,hajek2014,chen2015spectral,vojnovic2016,negahban2017}.

Conventional models of ranked data and aggregation algorithms that rely on them make the assumption that the data is either produced by a single user\footnote{We use the term user to refer to any entity that provides ranked data. In specific applications other terms may be more appropriate, such as voter, expert, judge, worker, and annotator.} or from a set of users that are similar. In real-world datasets, however, users that provide the raw data are usually diverse with different levels of familiarity with the objects of interest, thus providing data that is not uniformly reliable and should not have equal influence on the final result. This is of particular importance in applications such as aggregating expert opinions for decision-making and aggregating annotations provided by workers in crowd sourcing settings.

In this paper, we study the problem of rank aggregation for heterogeneous populations of users. We present a generalization of Thurstone's model, called the \emph{heterogeneous Thurstone model} (HTM), which allows users with different noise levels, as well as a certain class of adversarial users. Unlike previous efforts on rank aggregation for heterogeneous populations such as~\cite{chen2013pairwise,kumar2011}, the proposed model maintains the generality of Thurstone's framework and thus also extends its special cases such as BTL and PL models. We evaluate the performance of the method using simulated data for different noise distributions. We also demonstrate that the proposed aggregation algorithm outperforms the state-of-the-art method for real datasets on evaluating the difficulty of English text and comparing the population of a set of countries.

\noindent\textbf{Our Contributions:} Our main contributions are summarized as follows
\begin{itemize}
    \item We propose a general model called the heterogeneous Thurstone model (HTM) for producing ranked data based on heterogeneous sources, which reduces to the heterogeneous BTL (HBTL) model when the noise follows the Gumbel distribution and to the heterogeneous Thurstone Case V (HTCV) model when the noise follows the normal distribution respectively.
    \item We develop an efficient algorithm for aggregating pairwise comparisons and estimating user accuracy levels for a wide class of noise distributions based on minimizing the negative log-likelihood loss via alternating gradient descent.
    \item We theoretically show that the proposed algorithm converges to the unknown score vector and the accuracy vector at a locally linear rate up to a tight statistical error under  mild conditions.
    \item For models with specific noise distributions such as the HBTL and HTCV, we prove that the proposed algorithm converges linearly to the unknown score vector and accuracy vector up to statistical errors in the order of $O(n^2\log(mn^2)/(mk))$, where $k$ is sample size, $n$ is the number of items and $m$ is the number of users. When $m=1$, the statistical error matches the error bound in the state-of-the-art work for single user BTL model \citep{negahban2017}.
    \item We conduct thorough experiments on both synthetic and real world data to validate our theoretical results and demonstrate the superiority of our proposed model and algorithm.
\end{itemize}

The reminder of this paper is organized as follows. In Section \ref{sec:relatedWork}, we review the most related work in the literature. In Section~\ref{sec:model}, we propose a family of heterogeneous Thurstone models. In Section~\ref{sec:loss}, we propose an efficient algorithm for learning the ranking from pairwise comparisons. We theoretically analyze the convergence of the proposed algorithm in Section~\ref{sec:analysis}. Thorough experimental results are presented in Section~\ref{sec:expr} and Section~\ref{sec:conc} concludes the paper. 

\section{Additional Related Work}\label{sec:relatedWork}




The problem of rank aggregation has a long history, dating back to the works of \cite{deborda1781} and \cite{decondorcet1785} in the 18th century, where the problems of social choice and voting were discussed. 
More recently, the problem of aggregating pairwise comparisons, where comparisons are incorrect with a given probability $p$, was studied by~\cite{braverman2008} and~\cite{wauthier2013a}.
Instead of assuming the same probability for all comparisons to be incorrect, it is natural to assume that the comparison of similar items is more likely to be noisy than those items that are distinctly different. This intuition is reflected in the random utility model (RUM), also known as \textit{Thurstone's model}~\citep{thurstone1927}, where each item has a true score, and users provide rankings of subsets of items by comparing approximate version of these scores corrupted by additive noise.

When restricted to comparing pairs of items, Thurstone's model reduces to the BTL model \citep{zermelo1929,bradley1952,luce1959,hunter2004a} if the noise follows the Gumbel distribution, and to the Thurstone Case V (TCV) model~\citep{thurstone1927} if the noise is normally distributed. 
Recently, \citet{negahban2012} proposed Rank Centrality, an iterative method with a random walk interpretation and showed that it performs as well as the maximum likelihood (ML) solution \citep{zermelo1929,hunter2004a} for BTL models and provided non asymptotic performance guarantees. \citet{chen2015spectral} studied identifying the top-K candidates under the BTL model and its sample complexity.

Thurstone's model can also be used to describe data from comparisons of multiple items.
\citet{hajek2014} provided an upper bound on the error of the ML estimator and studied its optimality when data consists of partial rankings (as opposed to pairwise comparisons) under the PL model. \citet{yu2000} studied order statistics under the normal noise distribution with consideration of item confusion covariance and user perception shift in a Bayesian model. \citet{weng2011bayesian} proposed a Bayesian approximation method for game player ranking with results from two-team matches.
\citet{guiver2009} studied the ranking aggregation problem with partial ranking (PL model) in a Bayesian framework. However, due to the nature of Bayesian method, above mentioned work
provided few theoretical analysis. 
\citet{vojnovic2016} studied the parameter estimation problem for Thurstone models where first choices among a set of alternatives are observed. \citet{raman2014methods,raman2015bayesian} proposed the peer grading methods for solving a similar problem as ours, while the generative models to aggregate partial rankings and pairwise comparisons are completely different.
Very recently, \citet{zhao2018learning} proposed the $k$-RUM model which assumes that the rank distribution has a mixture of $k$ RUM components. They also provided the analyses of identifiability and efficiency of this model.

Almost all aforementioned works assume that all the data is provided by a single user or that all users have the same accuracy. However, this assumption is rarely satisfied in real-world datasets. The accuracy levels of different users are considered in \cite{kumar2011}, which assumes that each user is correct with a certain probability and studies the problem via simulation methods such as naive Bayes and majority voting.
In their pioneering work, \citet{chen2013pairwise} studied rank aggregation in a crowd-sourcing environment for pairwise comparisons, modeled via the BTL or TCV model, where noisy BTL comparisons are assumed to be further corrupted. They are flipped with a probability that depends on the identity of the worker. The $k$-RUM model proposed by \citet{zhao2018learning} considered a mixture of ranking distributions, without using extra information on who contributed the comparison, it may suffer from common mixture model issues.


\section{Modeling Heterogeneous Ranked Data}\label{sec:model}

Before introducing our Heterogeneous Thurstone Model, we start by providing some preliminaries of Thurstone's preference model in further detail. 
Consider a set of $n$ items. 
The score vector for the items is denoted by $\bs=\left(s_1,\dotsc,s_n\right)^\top$. These items/objects are evaluated by a set of $m$ independent users. Each user may be asked to express their preference concerning a subset of items $\{i_1,\dotsc,i_h\}\subseteq[n]$, where $2\le h \le n$. For each item $i$, the user first estimates an empirical score for it as
\begin{align}\label{eq:latent_uni_score}
    z_i=s_{i}+\epsilon_i,
\end{align}
where $\epsilon_i$ is a random noise introduced by this evaluation process. 
This coarse estimate of score $z_i$ is still implicit and cannot be queried or observed by the ranking algorithm. Instead, the user only produces a ranking of these $h$ items by sorting the scores $z_i$. We thus have
\begin{equation}\label{eq:data_model}
\Pr\left({\pi_1}\succ {\pi_2}\succ\dotsm\succ {\pi_h}\right) =
\Pr\left(z_{\pi_1}> z_{\pi_2}>\dotsm> z_{\pi_h}\right),
\end{equation}
where $i\succ j$ indicates that $i$ is preferred to $j$ by this user and $\{\pi_1,\ldots,\pi_h\}$ is a permutation of $\{i_1,\dotsc,i_h\}$. Each time item $i$ is compared with other items, a new score estimate $z_i$ is produced by the user for are commonly assumed to be i.i.d. 
\citep{braverman2008,negahban2012,wauthier2013a}. 


\subsection{The Heterogeneous Thurstone Model}\label{sec:HTM_model}
In real-world applications, users often have different levels of expertise and some may even be adversarial. Therefore, it is natural for us to propose an extension of the Thurstone's model presented above, referred to as the \emph{Heterogeneous Thurstone Model} (HTM), which has the flexibility to reflect the different levels of expertise of different users. Specifically, we assume that each user has a different level of making mistakes in evaluating items, i.e., the evaluation noise of user $u$ is controlled by a scaling factor $\gamma_{u}>0$. The proposed model is then represented as follows:
\begin{equation}\label{eq:latent}
z_i^u = s_{i} + \epsilon_i/\gamma_{u}.
\end{equation}
Based on the estimated scores of each user for each item, the probability of a certain ranking of $h$ items provided by user $u$ is again given by~\eqref{eq:data_model}. While this extension actually applies to both pairwise comparisons and multi-item orderings, we mainly focus on pairwise comparisons in this paper.

When two items $i$ and $j$ are compared by user $u$, we denote by $Y_{ij}^{u}$ the random variable representing the result,
\begin{equation}
Y_{ij}^{u} = \begin{cases}
1&\text{if } i\succ j;\\
0&\text{if } i\prec j.
\end{cases}
\end{equation}
Let $F$ denote the CDF of $\epsilon_j-\epsilon_i$, where $\epsilon_i$ and $\epsilon_j$ are two i.i.d.\ random variables. For the result $Y_{ij}^{u}$ of comparison of $i$ and $j$ by user $u$, we have
\begin{align}\label{eq:simple-comps}
\Pr(Y_{ij}^{u}=1;s_{i}, s_j,\gamma_{u}) = \Pr(\epsilon_j-\epsilon_i<\gamma_{u}(s_{i}-s_j))
&= F\left(\gamma_{u}(s_{i}-s_j)\right).
\end{align}
It is clear that the larger the value of $\gamma_{u}$, the more accurate the user is, since large $\gamma_{u}>0$ increases the probability of preferring an item with higher score to one with lower score.

We now consider several special cases arising from specific noise distributions. First, if $\epsilon_i$ follows a Gumbel distribution with mean 0 and scale parameter 1, then we obtain the following \emph{Heterogeneous BTL} (HBTL) model:
\begin{align}\label{eq:new_comp_Gumbel}
\log \Pr(Y_{ij}^{u}=1;s_{i}, s_j,\gamma_{u})
&= \log\frac{e^{\gamma_{u}s_{i}}}{e^{\gamma_{u}s_{i}}+e^{\gamma_{u}s_j}}=-\log(1+\exp({-\gamma_{u}(s_{i}-s_j)})),
\end{align}
which follows from the fact that the difference between two independent Gumbel random variables has the logistic distribution. We note that setting $\gamma_{u}=1$ recovers the traditional BTL model~\citep{bradley1952}.

If $\epsilon_i$ follows the standard normal distribution, we obtain the following \emph{Heterogeneous Thurstone Case V} (HTCV) model:
{\small
\begin{align}\label{eq:new_comp_normal}
\log\Pr(Y_{ij}^{u}=1;s_{i}, s_j,\gamma_{u}) &= \log\Phi\bigg({\frac{\gamma_{u}(s_{i}-s_j)}{\sqrt 2}}\bigg),
\end{align}}%
where $\Phi$ is the CDF of the standard normal distribution. Again, when $\gamma_{u}=1$, this reduces to Thurstone's Case V (TCV) model for pairwise comparisons~\citep{thurstone1927}.

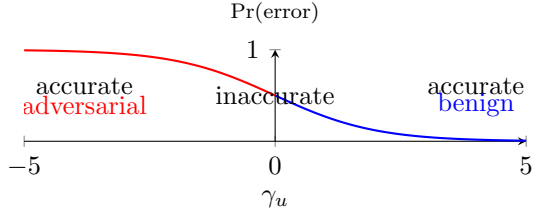
\begin{figure}[t]
    \begin{center}
        \begin{tikzpicture}
            \tikzstyle{every node}=[font=\small]
            \begin{axis}[
            no markers,
            xlabel = $\gamma_{u}$,
            title = {\scriptsize $\Pr(\text{error})$},
            ymin = 0, ymax=1,
            width = .50\textwidth, height = 2.8cm,
            axis y line=middle, axis x line=bottom,
            xtick = {-5,0,5},
            ytick = {0,1}
            ]
            \addplot[color=blue,domain=0:5,samples=1000,thick]{1-exp(x)/(1+exp(x))};
            \addplot[color=red,domain=-5:0,samples=1000,thick]{1-exp(x)/(1+exp(x))};
            \node at (axis cs:0,.5) [anchor=center] {inaccurate};
            \node at (axis cs:-3.8,.6) [anchor=center] {accurate};
            \node at (axis cs:-3.8,.4) [anchor=center] {\textcolor{red}{adversarial}};
            \node at (axis cs:4,.6) [anchor=center] {accurate};
            \node at (axis cs:4,.4) [anchor=center] {\textcolor{blue}{benign}};
            \end{axis}
        \end{tikzpicture}
    \end{center}
    \caption{The effect of $\gamma_{u}$ on the probability of error for a BTL comparison in which items have scores 0 and 1. In particular, for large negative values of $\gamma_{u}$, the user is accurate (with a high level of expertise) but adversarial.}
    \label{fig:gamma}
\end{figure}

\paragraph{Adversarial users:} Under our heterogeneous framework, we can also model a certain class of adversarial users, whose goal is to make the estimated ranking be the opposite of the true ranking, so that, for example, an inferior item is ranked higher than the alternatives. We assume for adversarial users, the score of item $i$ is $C-s_{i}$, for some constant $C$. Changing $s_{i}$ to $C-s_{i}$ in~\eqref{eq:simple-comps} is equivalent to assuming the user has a negative accuracy $\gamma_{u}$. In this way, the accuracy of the user is determined by the magnitude $|\gamma_{u}|$ and its trustworthiness by $\sign(\gamma_{u})$, as illustrated in Figure~\ref{fig:gamma}. When adversarial users are present, this will facilitate optimizing the loss function, since instead of solving the combinatorial optimization problem of deciding which users are adversarial, we simply optimize the value of $\gamma_{u}$ for each user.

One relevant work to ours is the CrowdBT algorithm proposed by \cite{chen2013pairwise}, where they also explored the accuracy level of different users in learning a global ranking. In particular, they assume that each user has a probability $\eta_u$ of making mistakes in comparing items $i$ and $j$:
\begin{align}\label{eq:crowdbt}
\Pr(Y_{ij}^{u}=1;s_{i}, s_j,\eta_u)
&= \eta_u\Pr(i \succ j)+(1-\eta_u)\Pr(j \succ i),
\end{align}
where $\Pr(i\succ j)$ and $\Pr(j \succ i)$ follow the BTL 
model.
This 
translates to introducing a parameter in the likelihood function to quantify the reliability of each pairwise comparison. This parameterization, however, deviates from the additive noise in Thurstonian models defined as in \eqref{eq:latent_uni_score} such as BTL and Thurstone's Case V. Specifically, the Thurstonian model explains the noise observed in pairwise comparisons as resulting from the additive noise in estimating the latent item scores. Therefore, the natural extension of Thurstonian models to a heterogeneous population of users is to allow different noise levels for different users, as was done in~\eqref{eq:latent}. As a result, CrowdBT cannot be easily extended to settings where more than two items are compared at a time. In contrast, the model proposed here is capable to describe such generalizations of Thurstonian models, such as the PL model.

\section{Optimization and Rank Aggregation}\label{sec:loss}

In this section, we define the pairwise comparison loss function for the population of users and propose an efficient and effective optimization algorithm to minimize it. We denote by $\bY^{u}$ the matrix containing all pairwise comparisons $Y_{ij}^{u}$ of user $u$ on items $i$ and $j$. The entries of $\bY^{u}$ are $0/1/?$, where $?$
indicates that the pair was not compared by the user. Furthermore, let $\cD_{u}$ denote the set of all pairs $(i,j)$ compared by user $u$. We define the loss function for each user $u$ as
\begin{align*}
    \begin{split}
\cL_{u}\left(\bs,\gamma_{u};\bY^{u}\right)&= - \frac{1}{k_{u}}{\sum_{\left(i,j\right)\in\mathcal{D}_{u} }}\log \Pr(Y_{ij}^{u}=1|s_{i}, s_j,\gamma_{u})\\
&= - \frac{1}{k_u}{\sum_{\left(i,j\right)\in\mathcal{D}_{u} }}\log  F\left(\gamma_{u}(s_{i}-s_j)\right),
    \end{split}
\end{align*}
where $k_u=|\cD_{u}|$ is the number of comparisons by user $u$. Then, the total loss function for $m$ users is
\begin{equation}\label{eq:total-loss}
    \cL\left(\bs,\bgamma;\bY\right) =\frac{1}{m}\sum_{u=1}^m \cL_{u}\left(\bs,\gamma_{u};\bY^{u}\right),
\end{equation} where $\bgamma=(\gamma_1,\dotsc,\gamma_m)^{\top}$ and $\bY=(\bY^1,\dotsc,\bY^m)$. We denote the unknown true score vector as $\bs^*$ and the true accuracy vector as $\bgamma^*$. Given observation $\bY$, our goal is to recover $\bs^*$ and $\bgamma^*$ via minimizing the loss function in \eqref{eq:total-loss}. To ensure the identifiability of $\bs^*$, we follow \cite{negahban2017} to assume that $\one^{\top}\bs^*=\sum_{i=1}^n s_i^{*}=0$, where $\one\in\RR^n$ is the all one vector. The following proposition shows that  the loss function $\cL$ is convex in $\bs$ and in $\bgamma$ separately if the PDF of $\epsilon_i$ is log-concave.
\begin{proposition}\label{prop:loss_convex_density_logconcave}
If the distribution of the noise $\epsilon_i$ in~\eqref{eq:latent} is log-concave, then the loss function $\cL(\bs,\bgamma;\bY)$ given in~\eqref{eq:total-loss} is convex in $\bs$, and in $\bgamma$ respectively.
\end{proposition}
The log-concave family includes many well-known distributions such as normal, exponential, Gumbel, gamma and beta distributions. In particular, the noise distributions used in BTL and Thurstone's Case V (TCV) models fall into this category. Although the loss function $\cL$ is non convex with respect to the joint variable $(\bs,\bgamma)$, Proposition \ref{prop:loss_convex_density_logconcave} inspires us to perform alternating gradient descent \citep{jain2013low} on $\bs$ and $\bgamma$ to minimize the loss function. As is shown in Algorithm~\ref{alg:PGD}, we alternating perform gradient descent update on $\bs$ (or $\bgamma$) while fixing $\bgamma$ (or $\bs$) at each iteration. 
In addition to the alternating gradient descent steps, we shift $\bs^{(t)}$ in Line \ref{algline:scale} of Algorithm \ref{alg:PGD} such that $\one^{\top}\bs^{(t)}=0$ to avoid the aforementioned identifiability issue of $\bs^*$. After $T$ iterations, given the output $\bs^{(T)}$, the estimated ranking of the items is obtained by sorting $\{s^{(T)}_1,\ldots,s^{(T)}_n\}$ in descending order (item with the highest score in $\bs^{(T)}$ is the most preferred). 

\begin{figure}[hb]
    \begin{algorithm}[H]
    	\caption{HTMs with Alternating Gradient Descent}\label{alg:PGD}
    	\begin{algorithmic}[1]
    		\STATE \textbf{input:} learning rates $\eta_1,\eta_2>0$, initial points $\bs^{(0)}$ and $\bgamma^{(0)}$ satisfying $\|\bs^{(0)}-\bs^*\|_2^2+\|\bgamma^{(0)} -\bgamma^{*}\|_2^2\leq r$, number of iteration $T$, comparison results by users $\bY$.
    		\FOR{$t=0,\ldots,T-1$}
        		\STATE $\tilde\bs^{(t+1)}=\bs^{(t)}-\eta_1 \nabla_{\bs} \cL\big(\bs^{(t)},\bgamma^{(t)};\bY\big)$
        		    \label{line:1}
        		\STATE $\bs^{(t+1)}=(\Ib-\one\one^{\top}/n)\tilde\bs^{(t+1)}$
        		    \label{algline:scale}
        		\STATE $\bgamma^{(t+1)}=\bgamma^{(t)}-\eta_2 \nabla_{\bgamma}\cL\big(\bs^{(t)},\bgamma^{(t)};\bY\big)$
        		    \label{line:2}
    		\ENDFOR
    		\STATE \textbf{output:} $\bs^{(T)}$, $\bgamma^{(T)}$.
    	\end{algorithmic}
    \end{algorithm}
\end{figure}

As we will show in the next section, the convergence of Algorithm \ref{alg:PGD} to the optimal points $\bs^*$ and $\bgamma^*$ is guaranteed if an initialization such that $\bs^{(0)}$ and $\bgamma^{(0)}$ are close to the unknown parameters is available. In practice, to initialize $\bs$, we can use the solution provided by the rank centrality algorithm~\citep{negahban2012} or start from uniform or random scores. 
In this paper, we initialize $\bs$ and $\bgamma$, as $\bs^{(0)}=\one$ and $\bgamma^{(0)}=\one$. We note that multiplying $\bs$ or $\bgamma$ by a negative constant does not alter the loss but reverses the estimated ranking. Implicit in our initialization is the assumption that the majority of the users are trustworthy and thus have positive $\gamma$.
When data is sparse, there may be subsets of items that are not compared directly or indirectly. In such cases, regularization may be necessary, which is discussed in further detail in Section~\ref{sec:expr}.

\section{Theoretical Analysis of the Proposed Algorithm}\label{sec:analysis}
In this section, we provide the convergence analysis of Algorithm \ref{alg:PGD} for the general loss function defined in \eqref{eq:total-loss}. Without loss of generality, we assume the number of observations $k_u=k$ for all users $u\in[m]$ throughout our analysis. Since there's no specific requirement on the noise distributions in the general HTM model, to derive the linear convergence rate, we need the following conditions on the loss function $\cL$, which are standard in the literature of alternating minimization \citep{jain2013low,zhu2017high,xu2017efficient,xu2017speeding,zhang2018unified,chen2018covariate}. Note that all these conditions can actually be verified once we specify the noise distribution in specific models. We provide the justifications of these conditions in the appendix.

\begin{condition}[Strong Convexity]\label{assump:strong_convex}
$\cL$ is $\mu_1$-strongly convex with respect to $\bs\in\RR^n$ and $\mu_2$-strongly convex with respect to $\bgamma\in\RR^m$. In particular, there is a constant $\mu_1>0$ such that for all $\bs,\bs'\in\RR^n$,
\begin{align*}
    \cL(\bs,\bgamma)&\geq\cL(\bs',\bgamma)+\la\nabla_{\bs}\cL(\bs',\bgamma),\bs-\bs'\ra+\mu_1/2\|\bs-\bs'\|_2^2.
\end{align*}
And there is a constant $\mu_2>0$ such that for all $\bgamma,\bgamma'\in\RR^m$, it holds
\begin{align*}
    \cL(\bs,\bgamma)&\geq\cL(\bs,\bgamma')+\la\nabla_{\bgamma}\cL(\bs,\bgamma'),\bgamma-\bgamma'\ra+\mu_2/2\|\bgamma-\bgamma'\|_2^2.
\end{align*}
\end{condition}

\begin{condition}[Smoothness]\label{assump:smooth}
$\cL$ is $L_1$-smooth with respect to $\bs\in\RR^n$ and $L_2$-smooth with respect to $\bgamma\in\RR^m$. In particular, there is a constant $L_1>0$ such that for all $\bs,\bs'\in\RR^n$, it holds
\begin{align*}
    \cL(\bs,\bgamma)&\leq\cL(\bs',\bgamma)+\la\nabla_{\bs}\cL(\bs',\bgamma),\bs-\bs'\ra+L_1/2\|\bs-\bs'\|_2^2.
\end{align*}
And there is a constant $L_2>0$ such that for all $\bgamma,\bgamma'\in\RR^m$, it holds
\begin{align*}
    \cL(\bs,\bgamma)&\leq\cL(\bs,\bgamma')+\la\nabla_{\bgamma}\cL(\bs,\bgamma'),\bgamma-\bgamma'\ra+L_2/2\|\bgamma-\bgamma'\|_2^2.
\end{align*}
\end{condition}
The next condition is a variant of the usual Lipschitz gradient condition. It is worth noting that the gradient is derived with respect to $\bs$ (or $\bgamma$), while the upper bound is the difference of $\bgamma$ (or $\bs$). This condition is commonly imposed and verified in the analysis of expectation-maximization algorithms \citep{wang2015high} and alternating minimization \citep{jain2013low}.
\begin{condition}[First-order Stability]\label{lemma:FOS}
There are constants $M_1,M_2>0$ such that $\cL$ satisfies
\begin{align*}
    \|\nabla_{\bs} \cL(\bs,\bgamma)-\nabla_{\bs} \cL(\bs,\bgamma')\|_2
    &\leq M_1\|\bgamma-\bgamma'\|_2,\\
    \|\nabla_{\bgamma} \cL(\bs,\bgamma)-\nabla_{\bgamma} \cL(\bs',\bgamma)\|_2
    &\leq M_2\|\bs-\bs'\|_2,
\end{align*}
for all $\bs,\bs'\in\RR^n$ and $\bgamma,\bgamma'\in\RR^m$.
\end{condition}

Note that the loss function in \eqref{eq:total-loss} is defined based on finitely many samples of observations. The next condition shows how close the gradient of the sample loss function is to the expected loss function.
\begin{condition}\label{lemma:staterr_general}
Denote $\bar\cL$ as the expected loss, where the expectation of $\cL$ is taken over the random choice of the comparison pairs and the observation $\bY$. With probability at least $1-1/n$, we have
\begin{align*}
    \|\gradsb\cL(\bs,\bgamma)-\gradsb\bar\cL(\bs,\bgamma)\|_2&\leq\epsilon_{1}(k,n),\\
    \|\gradgam\cL(\bs,\bgamma)-\gradgam\bar\cL(\bs,\bgamma)\|_2&\leq\epsilon_{2}(k,n),
\end{align*}
where $n$ is the number of items and $k$ is the number of observations for each user. In addition, $\epsilon_1(k,n)$ and $\epsilon_2(k,n)$ will go to zero when sample size $k$ goes to infinity.
\end{condition}
$\epsilon_1(k,n)$ and $\epsilon_2(k,n)$ in Condition \ref{lemma:staterr_general} are also called the statistical errors \citep{wang2015high,xu2017speeding} between the sample version gradient and the expected (population) gradient.

Now we deliver our main theory on the linear convergence of Algorithm \ref{alg:PGD} for general HTM models. Full proofs can be found in the appendix.
\begin{theorem}\label{thm:convergence_general}
For a general HTM model, assume Conditions \ref{assump:strong_convex}, \ref{assump:smooth}, \ref{lemma:FOS} and \ref{lemma:staterr_general} hold and that $M_1,M_2\leq\sqrt{\mu_1\mu_2}/4$. Denote that $\|\bs^*\|_{\infty}= \smax$ and $\|\bgamma^*\|_{\infty}=\gammamax$. Suppose the initialization guarantees that
$\|\bs^{(0)}-\bs^*\|_2^2+\|\bgamma^{(0)}-\bgamma^*\|_2^2\leq r^2$,where $r=\min\{\mu_1/(2M_1),\mu_2/(2M_2)\}$. If we set the step size $\eta_1=\eta_2=\mu/(12(L^2+M^2))$, where $L=\max\{L_1,L_2\}$, $\mu=\min\{\mu_1,\mu_2\}$ and $M=\max\{M_1,M_2\}$, then the output of Algorithm \ref{alg:PGD} satisfies
\begin{align*}
    \|\bs^{(T)}-\bs^*\|_2^2+\|\bgamma^{(T)}-\bgamma^*\|_2^2
    &\leq r^2\rho^T+\frac{\epsilon_1(k,n)^2+\epsilon_2(k,n)^2}{\mu^2}
\end{align*}
with probability at least $1-1/n$, where the contraction parameter is $\rho=1-\mu^2/(48(L^2+M^2))$.
\end{theorem}
\begin{remark}
Theorem \ref{thm:convergence_general} establishes the linear convergence of Algorithm \ref{alg:PGD} when the initial points are close to the unknown parameters. The first term on the right-hand side is called the optimization error, which goes to zero as iteration number $t$ goes to infinity. The second term is called the statistical error of the HTM model, which goes to zero when sample size $mk$ goes to infinity. Hence, the estimation error of our proposed algorithm converges to the order of $O((\epsilon_1(k,n)^2+\epsilon_2(k,n)^2)/\mu^2)$ after $t=O(\log((\epsilon_1(k,n)^2+\epsilon_2(k,n)^2)/\mu^2r^2)/\log\rho)$ iterations.
\end{remark}
Note that the results in Theorem \ref{thm:convergence_general} hold for any general HTM models with Algorithm \ref{alg:PGD} as a solver. In particular, if we run the alternating gradient descent algorithm on the HBTL and HTCV models proposed in Section \ref{sec:model}, we will also obtain linear convergence rate to the true parameters up to a statistical error in the order of $O(n^2\log(mn^2)/(mk))$, which matches the state-of-the-art statistical error for such models \citep{negahban2017}. We provide the implications of Theorem \ref{thm:convergence_general} on specific models in the supplementary material.

\section{Experiments}\label{sec:expr}
In this section, we present experimental results to show the performance of the proposed algorithm on heterogeneous populations of users. The experiments are conducted on both synthetic and real data with both benign users and adversarial users. We use the Kendall's tau correlation~\cite{kendall1948} between the estimated and true rankings to measure the similarity between rankings, which is defined as $\tau = \frac{2(c-d)}{n(n-1)}$, where $c$ and $d$ are the number of pairs on which the two rankings agree and disagree, respectively. Pairs that are tied in at least one of the rankings are not counted in $c$ or $d$.


\textbf{Baseline methods:} In Gumbel noise setting, we compare Algorithm \ref{alg:PGD} based on our proposed HBTL model with (i) the BTL model that can be optimized through iterative maximum-likelihood methods \citep{negahban2012} or spectral methods such as Rank Centrality \citep{negahban2017}; and (ii) the CrowdBT algorithm~\citep{chen2013pairwise}, which is a variation of BTL that allows users with different levels of accuracy. In the normal noise setting, we compare Algorithm \ref{alg:PGD} based on our proposed HTCV model with TCV model. We also implemented a TCV equivalent of CrowdBT and report its performance as CrowdTCV.

\subsection{Experimental Results on Synthetic Data}
We set number of items $n=20$, number of users $m=9$ and set the ground truth score vector $\bs$ to be uniformly distributed in $[0,1]$. The $m$ users are divided into groups $A$ and $B$, consisting of 3 and 6 users respectively. These two groups of users generate heterogeneous data in the sense that users in group $A$ are more accurate than those in group $B$. 
We vary $\gamma_A$ in the range of $\{2.5,5,10\}$ and $\gamma_B$ in the range of $\{0.25,1,2.5\}$, which leads to in total $9$ configurations of data generation. For each configuration, we conduct the experiment under the following two settings:
\begin{itemize}[leftmargin=2em,nosep]
    \item[(1)] \textbf{Benign:} $\gamma_1,\ldots,\gamma_3=\gamma_A$ (Group A); $\gamma_4,\ldots,\gamma_9=\gamma_B$ (Group B).
    \item[(2)] \textbf{Adversarial:} $\gamma_1=-\gamma_A$, $\gamma_2,\gamma_3=\gamma_A$ (Group A); $\gamma_4,\gamma_5=-\gamma_B$, $\gamma_6,\ldots,\gamma_9=\gamma_B$ (Group B).
\end{itemize}
We also test on various densities of compared pairs, which effectively controls the sample size. In particular, we choose 4 sets of $\alpha$, which denote the portion of all possible pairs that are compared. The larger the value, the more pairs are compared by each user. The simulation process is as follows: we first generate $n(n-1)$ ordered pairs of items, where $n$ is the number of items. This is equivalent to comparing each unique pair of items twice. Then for each pair of items, response from every annotator had a probability of $\alpha$ to be recorded and used for training the model. And $\alpha$ is chosen from $\{0.2, 0.4, 0.6, 0.8\}$ to make up for four runs. Each experiment is repeated $100$ times with different random seeds.



Under setting (1), 
we plot the estimation error of Algorithm \ref{alg:PGD} v.s. number of iterations for HBTL and HTCV model in Figures \ref{fig:est_err_s_hbtl}-\ref{fig:est_err_ga_hbtl} and \ref{fig:est_err_s_htcv}-\ref{fig:est_err_ga_htcv} respectively. In all settings, our algorithm enjoys a linear convergence rate to the true parameters up to statistical errors, which is well aligned with the theoretical results in Theorem \ref{thm:convergence_general}. 
\begin{figure}[htbp]
    \centering
    \subfigure[Estimation error for $\bs^*$]{
        \includegraphics[height=0.17\textwidth]{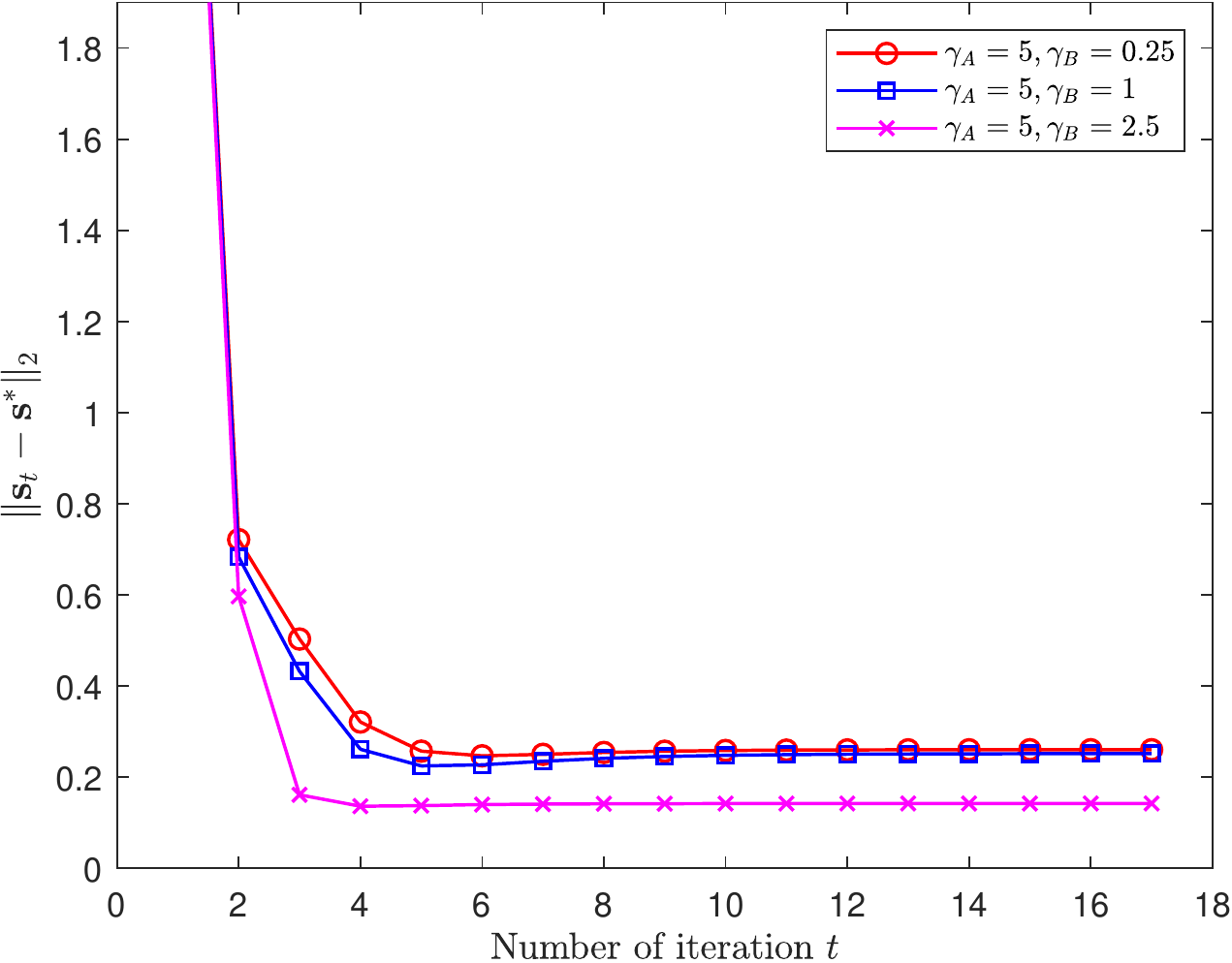}\label{fig:est_err_s_hbtl}
    }
    \subfigure[Estimation error for $\bgamma^*$]{
        \includegraphics[height=0.17\textwidth]{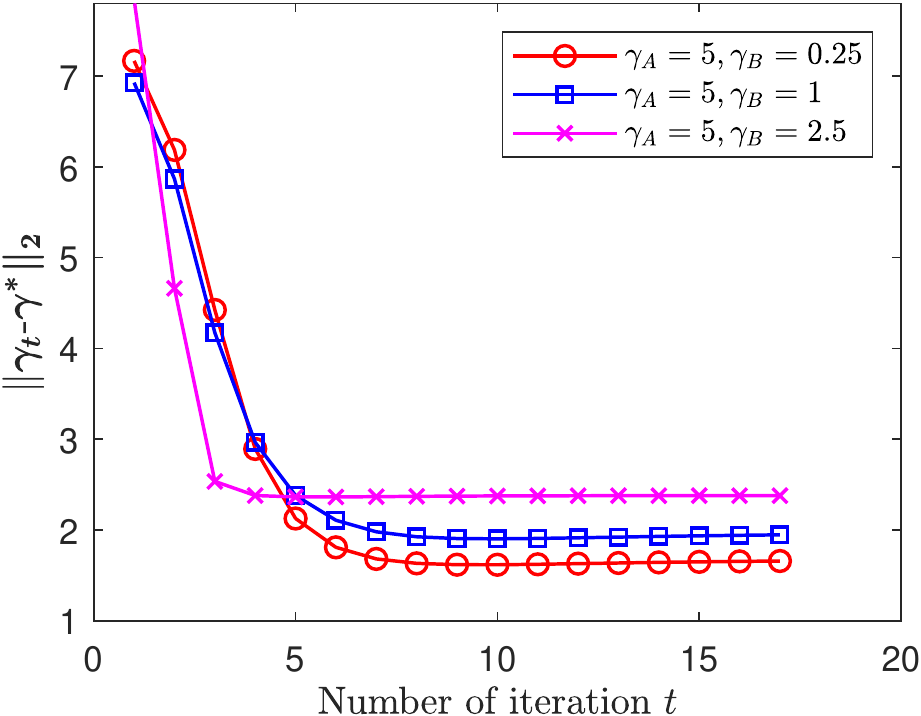}\label{fig:est_err_ga_hbtl}
    }
    \subfigure[Estimation error for $\bs^*$]{
        \includegraphics[height=0.17\textwidth]{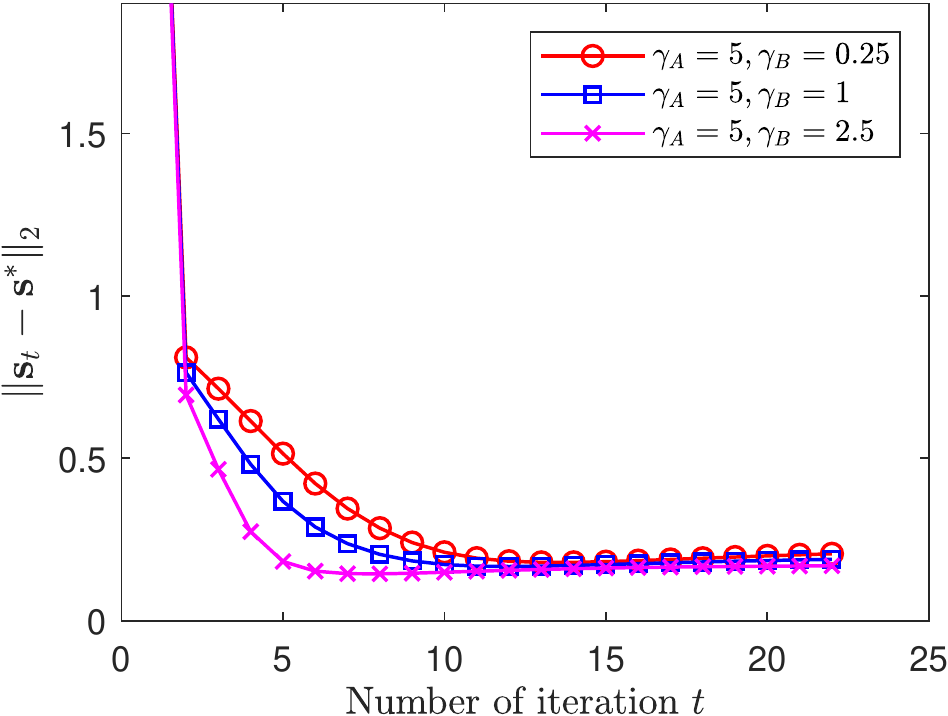}\label{fig:est_err_s_htcv}
        }
    \subfigure[Estimation error for $\bgamma^*$]{
        \includegraphics[height=0.17\textwidth]{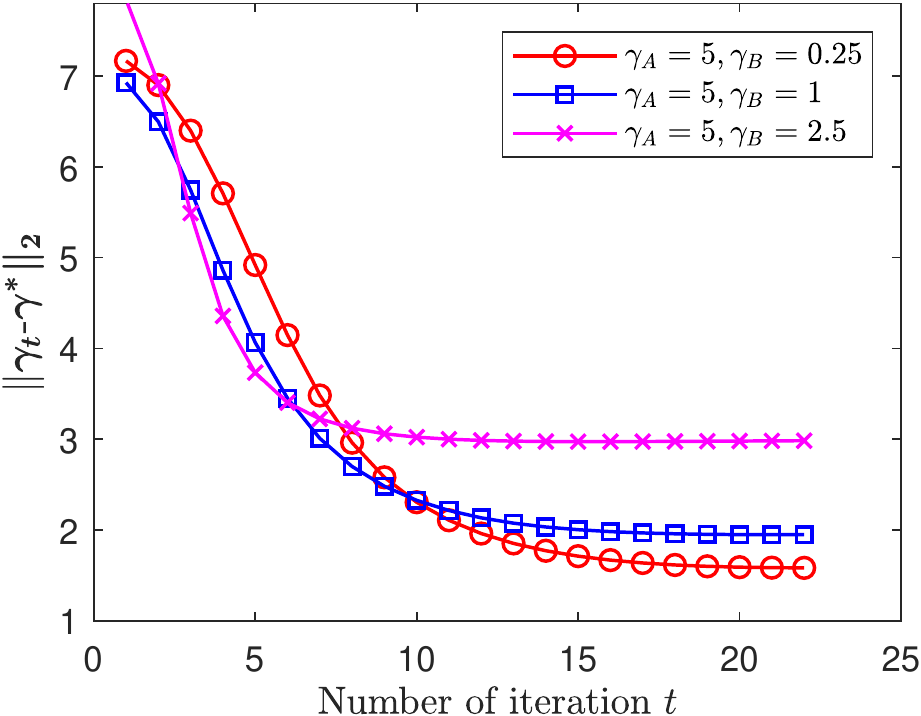}\label{fig:est_err_ga_htcv}
        }
    \caption{Evolution of estimation errors vs. number of iterations $t$ for HBTL model. (c)-(d): Evolution of estimation errors vs. number of iterations $t$ for HTCV model.
    }
    \label{fig}
\end{figure}

When there is no adversarial users in the system, the ranking results for Gumbel noises under different configurations of $\gamma_A$ and $\gamma_B$ are shown in Table~\ref{tab:simulation_benign_extended_gumbel} and the ranking results for normal noises under different configurations of $\gamma_A$ and $\gamma_B$ are shown in Table~\ref{tab:simulation_benign_extended_normal}. In both tables, each cell presents the Kendall's tau correlation between the aggregated ranking and the ground truth, averaged over 100 trials. For each experimental setting, we use the bold text to denote the method which achieved highest performance. We also underline the highest score whenever there is a tie.
It can be observed that in almost all cases, HBTL provides much more accurate rankings than BTL and HTCV significantly outperforms TCV as well.
In particular, the larger the difference between $\gamma_A$ and $\gamma_B$ is, the more significant the improvement is. The only exception is when $\gamma_A=\gamma_B=2.5$, in which case the data is not heterogeneous and our HTM model has no advantage. Nevertheless, our method still achieve comparable performance as BTL for non-heterogeneous data. It can also be observed that HBTL generally outperforms CrowdBT. But the advantage is not large, as CrowdBT 
also includes the different accuracy levels of different users. Importantly, however, as discussed in Section \ref{sec:HTM_model}, CrowdBT is not compatible with the additive noise in Thurstonian models and cannot be extended in a natural way to ranked data other than pairwise comparison. In addition, unlike CrowdBT, our method enjoys strong theoretical guarantees while maintaining a good performance. Tables \ref{tab:simulation_benign_extended_gumbel} and \ref{tab:simulation_benign_extended_normal} also illustrate an important fact: If there are users with high accuracy, the presence of low quality data does not significantly impact the performance of Algorithm \ref{alg:PGD}.
\begin{table}[!htbp]
    \centering
    \caption{Kendall's tau correlation for different method under Gumbel noise. Group $A$ users all have the accuracy level $\gamma_A$ and Group $B$ users all have the accuracy level $\gamma_B$. $\alpha$ represents the portion of all possible pairwise comparisons each annotator labeled in the simulation. The \textbf{bold} number highlights the highest performance and the \underline{underlined} number indicates a tie.
    \label{tab:simulation_benign_extended_gumbel}}
        \vspace{0.2in}
    \begin{small}
    \begin{tabular}{cccccc}
      \toprule
      \multirow{2}{*}{\shortstack[t]{Observ.\\Ratio}}&\multirow{2}{*}{$\gamma_B$}&\multirow{2}{*}{Methods}&\multicolumn{3}{c}{$\gamma_A$}\\
      \cmidrule(lr){4-6}
      &&& 2.5 & 5 & 10  \\
      \midrule
      \multirow{9}{*}{$\alpha = 0.8$}
&\multirow{3}{*}{0.25} &BTL& \small{0.767}\small{$\pm$}\small{0.055}&\small{0.836}\small{$\pm$}\small{0.043}&\small{0.879}\small{$\pm$}\small{0.032} \\ &&CrowdBT& \small{0.847}\small{$\pm$}\small{0.042}&\small{0.928}\small{$\pm$}\small{0.023}&\small{0.962}\small{$\pm$}\small{0.016} \\ &&HBTL& \textbf{\small{0.850}}\small{$\pm$}\small{0.041}&\textbf{\small{0.930}}\small{$\pm$}\small{0.024}&\textbf{\small{0.964}}\small{$\pm$}\small{0.015} \\\cmidrule(lr){2-6}
&\multirow{3}{*}{1.0} &BTL& \small{0.863}\small{$\pm$}\small{0.036}&\small{0.896}\small{$\pm$}\small{0.028}&\small{0.923}\small{$\pm$}\small{0.026} \\ &&CrowdBT& \underline{\small{0.875}}\small{$\pm$}\small{0.033}&\underline{\small{0.930}}\small{$\pm$}\small{0.024}&\small{0.967}\small{$\pm$}\small{0.018} \\ &&HBTL& \underline{\small{0.875}}\small{$\pm$}\small{0.033}&\underline{\small{0.930}}\small{$\pm$}\small{0.024}&\textbf{\small{0.969}}\small{$\pm$}\small{0.017} \\\cmidrule(lr){2-6}
&\multirow{3}{*}{2.5} &BTL& \textbf{\small{0.933}}\small{$\pm$}\small{0.022}&\small{0.946}\small{$\pm$}\small{0.019}&\small{0.959}\small{$\pm$}\small{0.018} \\ &&CrowdBT& \small{0.931}\small{$\pm$}\small{0.024}&\small{0.947}\small{$\pm$}\small{0.019}&\small{0.967}\small{$\pm$}\small{0.017} \\ &&HBTL& \small{0.931}\small{$\pm$}\small{0.025}&\textbf{\small{0.948}}\small{$\pm$}\small{0.021}&\textbf{\small{0.972}}\small{$\pm$}\small{0.015} \\\midrule

\multirow{9}{*}{$\alpha = 0.6$}
&\multirow{3}{*}{0.25} &BTL& \small{0.743}\small{$\pm$}\small{0.064}&\small{0.814}\small{$\pm$}\small{0.048}&\small{0.853}\small{$\pm$}\small{0.037} \\ &&CrowdBT& \small{0.823}\small{$\pm$}\small{0.050}&\textbf{\small{0.909}}\small{$\pm$}\small{0.034}&\small{0.954}\small{$\pm$}\small{0.018} \\ &&HBTL& \textbf{\small{0.824}}\small{$\pm$}\small{0.051}&\small{0.908}\small{$\pm$}\small{0.033}&\textbf{\small{0.955}}\small{$\pm$}\small{0.018} \\\cmidrule(lr){2-6}
&\multirow{3}{*}{1.0} &BTL& \small{0.837}\small{$\pm$}\small{0.036}&\small{0.872}\small{$\pm$}\small{0.033}&\small{0.903}\small{$\pm$}\small{0.033} \\ &&CrowdBT& \textbf{\small{0.853}}\small{$\pm$}\small{0.035}&\small{0.911}\small{$\pm$}\small{0.031}&\small{0.955}\small{$\pm$}\small{0.018} \\ &&HBTL& \small{0.851}\small{$\pm$}\small{0.033}&\textbf{\small{0.913}}\small{$\pm$}\small{0.028}&\textbf{\small{0.958}}\small{$\pm$}\small{0.017} \\\cmidrule(lr){2-6}
&\multirow{3}{*}{2.5} &BTL& \textbf{\small{0.913}}\small{$\pm$}\small{0.032}&\small{0.931}\small{$\pm$}\small{0.024}&\small{0.948}\small{$\pm$}\small{0.021} \\ &&CrowdBT& \small{0.910}\small{$\pm$}\small{0.028}&\small{0.935}\small{$\pm$}\small{0.020}&\small{0.961}\small{$\pm$}\small{0.016} \\ &&HBTL& \small{0.912}\small{$\pm$}\small{0.029}&\textbf{\small{0.936}}\small{$\pm$}\small{0.022}&\textbf{\small{0.967}}\small{$\pm$}\small{0.017} \\\midrule

\multirow{9}{*}{$\alpha = 0.4$}
&\multirow{3}{*}{0.25} &BTL& \small{0.671}\small{$\pm$}\small{0.062}&\small{0.761}\small{$\pm$}\small{0.053}&\small{0.812}\small{$\pm$}\small{0.048} \\ &&CrowdBT& \small{0.764}\small{$\pm$}\small{0.065}&\small{0.872}\small{$\pm$}\small{0.037}&\small{0.933}\small{$\pm$}\small{0.024} \\ &&HBTL& \textbf{\small{0.769}}\small{$\pm$}\small{0.061}&\textbf{\small{0.873}}\small{$\pm$}\small{0.034}&\textbf{\small{0.934}}\small{$\pm$}\small{0.022} \\\cmidrule(lr){2-6}
&\multirow{3}{*}{1.0} &BTL& \small{0.791}\small{$\pm$}\small{0.051}&\small{0.844}\small{$\pm$}\small{0.043}&\small{0.866}\small{$\pm$}\small{0.035} \\ &&CrowdBT& \small{0.798}\small{$\pm$}\small{0.050}&\small{0.889}\small{$\pm$}\small{0.029}&\small{0.934}\small{$\pm$}\small{0.027} \\ &&HBTL& \textbf{\small{0.806}}\small{$\pm$}\small{0.051}&\textbf{\small{0.891}}\small{$\pm$}\small{0.031}&\textbf{\small{0.936}}\small{$\pm$}\small{0.026} \\\cmidrule(lr){2-6}
&\multirow{3}{*}{2.5} &BTL& \textbf{\small{0.882}}\small{$\pm$}\small{0.034}&\small{0.910}\small{$\pm$}\small{0.030}&\small{0.919}\small{$\pm$}\small{0.027} \\ &&CrowdBT& \small{0.879}\small{$\pm$}\small{0.034}&\small{0.912}\small{$\pm$}\small{0.026}&\small{0.943}\small{$\pm$}\small{0.022} \\ &&HBTL& \small{0.880}\small{$\pm$}\small{0.032}&\textbf{\small{0.916}}\small{$\pm$}\small{0.028}&\textbf{\small{0.945}}\small{$\pm$}\small{0.020} \\\midrule

\multirow{9}{*}{$\alpha = 0.2$}
&\multirow{3}{*}{0.25} &BTL& \small{0.575}\small{$\pm$}\small{0.095}&\small{0.663}\small{$\pm$}\small{0.078}&\small{0.712}\small{$\pm$}\small{0.069} \\ &&CrowdBT& \small{0.644}\small{$\pm$}\small{0.094}&\small{0.798}\small{$\pm$}\small{0.055}&\textbf{\small{0.884}}\small{$\pm$}\small{0.035} \\ &&HBTL& \textbf{\small{0.665}}\small{$\pm$}\small{0.090}&\textbf{\small{0.805}}\small{$\pm$}\small{0.051}&\small{0.882}\small{$\pm$}\small{0.034} \\\cmidrule(lr){2-6}
&\multirow{3}{*}{1.0} &BTL& \textbf{\small{0.708}}\small{$\pm$}\small{0.073}&\small{0.768}\small{$\pm$}\small{0.057}&\small{0.804}\small{$\pm$}\small{0.039} \\ &&CrowdBT& \small{0.696}\small{$\pm$}\small{0.081}&\small{0.813}\small{$\pm$}\small{0.052}&\small{0.876}\small{$\pm$}\small{0.034} \\ &&HBTL& \small{0.702}\small{$\pm$}\small{0.079}&\textbf{\small{0.819}}\small{$\pm$}\small{0.052}&\textbf{\small{0.882}}\small{$\pm$}\small{0.034} \\\cmidrule(lr){2-6}
&\multirow{3}{*}{2.5} &BTL& \textbf{\small{0.820}}\small{$\pm$}\small{0.044}&\underline{\small{0.861}}\small{$\pm$}\small{0.043}&\small{0.883}\small{$\pm$}\small{0.033} \\ &&CrowdBT& \small{0.803}\small{$\pm$}\small{0.048}&\small{0.857}\small{$\pm$}\small{0.037}&\small{0.898}\small{$\pm$}\small{0.030} \\ &&HBTL& \small{0.807}\small{$\pm$}\small{0.049}&\underline{\small{0.861}}\small{$\pm$}\small{0.038}&\textbf{\small{0.904}}\small{$\pm$}\small{0.029} \\
\bottomrule
    \end{tabular}
    \end{small}
\end{table}

\begin{table}[!htbp]
    \centering
    \caption{Kendall's tau correlation for different methods under noise from the normal distribution. Group $A$ users all have the accuracy level $\gamma_A$ and Group $B$ users all have the accuracy level $\gamma_B$. $\alpha$ represents the portion of all possible pairwise comparisons each annotator labeled in the simulation. The \textbf{bold} number highlights the highest performance and the \underline{underlined} number indicates a tie.
    \label{tab:simulation_benign_extended_normal}} \vspace{0.2in}
    \begin{small}
    \begin{tabular}{cccccc}
      \toprule
      \multirow{2}{*}{\shortstack[t]{Observ.\\Ratio}}&\multirow{2}{*}{$\gamma_B$}&\multirow{2}{*}{Methods}&\multicolumn{3}{c}{$\gamma_A$}\\
      \cmidrule(lr){4-6}
      &&& 2.5 & 5 & 10  \\
      \midrule
\multirow{9}{*}{$\alpha = 0.8$}
&\multirow{3}{*}{0.25} &TCV& \small{0.811}\small{$\pm$}\small{0.048}&\small{0.860}\small{$\pm$}\small{0.040}&\small{0.885}\small{$\pm$}\small{0.036} \\ &&CrowdTCV& \small{0.881}\small{$\pm$}\small{0.032}&\underline{\small{0.943}}\small{$\pm$}\small{0.021}&\underline{\small{0.971}}\small{$\pm$}\small{0.014} \\ &&HTCV& \textbf{\small{0.882}}\small{$\pm$}\small{0.030}&\underline{\small{0.943}}\small{$\pm$}\small{0.021}&\underline{\small{0.971}}\small{$\pm$}\small{0.015} \\\cmidrule(lr){2-6}
&\multirow{3}{*}{1.0} &TCV& \small{0.885}\small{$\pm$}\small{0.036}&\small{0.910}\small{$\pm$}\small{0.027}&\small{0.925}\small{$\pm$}\small{0.029} \\ &&CrowdTCV& \underline{\small{0.897}}\small{$\pm$}\small{0.030}&\underline{\small{0.944}}\small{$\pm$}\small{0.020}&\small{0.973}\small{$\pm$}\small{0.015} \\ &&HTCV& \underline{\small{0.897}}\small{$\pm$}\small{0.033}&\underline{\small{0.944}}\small{$\pm$}\small{0.020}&\textbf{\small{0.975}}\small{$\pm$}\small{0.013} \\\cmidrule(lr){2-6}
&\multirow{3}{*}{2.5} &TCV& \underline{\small{0.945}}\small{$\pm$}\small{0.021}&\small{0.956}\small{$\pm$}\small{0.018}&\small{0.965}\small{$\pm$}\small{0.018} \\ &&CrowdTCV& \underline{\small{0.945}}\small{$\pm$}\small{0.021}&\small{0.954}\small{$\pm$}\small{0.019}&\small{0.976}\small{$\pm$}\small{0.014} \\ &&HTCV& \small{0.944}\small{$\pm$}\small{0.021}&\textbf{\small{0.959}}\small{$\pm$}\small{0.017}&\textbf{\small{0.981}}\small{$\pm$}\small{0.014} \\\midrule

\multirow{9}{*}{$\alpha = 0.6$}
&\multirow{3}{*}{0.25} &TCV& \small{0.763}\small{$\pm$}\small{0.059}&\small{0.830}\small{$\pm$}\small{0.043}&\small{0.850}\small{$\pm$}\small{0.041} \\ &&CrowdTCV& \small{0.845}\small{$\pm$}\small{0.038}&\textbf{\small{0.926}}\small{$\pm$}\small{0.023}&\underline{\small{0.961}}\small{$\pm$}\small{0.020} \\ &&HTCV& \textbf{\small{0.846}}\small{$\pm$}\small{0.040}&\small{0.925}\small{$\pm$}\small{0.025}&\underline{\small{0.961}}\small{$\pm$}\small{0.020} \\\cmidrule(lr){2-6}
&\multirow{3}{*}{1.0} &TCV& \small{0.862}\small{$\pm$}\small{0.038}&\small{0.892}\small{$\pm$}\small{0.034}&\small{0.912}\small{$\pm$}\small{0.025} \\ &&CrowdTCV& \small{0.870}\small{$\pm$}\small{0.035}&\small{0.930}\small{$\pm$}\small{0.028}&\small{0.962}\small{$\pm$}\small{0.019} \\ &&HTCV& \textbf{\small{0.875}}\small{$\pm$}\small{0.033}&\textbf{\small{0.932}}\small{$\pm$}\small{0.027}&\textbf{\small{0.963}}\small{$\pm$}\small{0.018} \\\cmidrule(lr){2-6}
&\multirow{3}{*}{2.5} &TCV& \textbf{\small{0.927}}\small{$\pm$}\small{0.027}&\small{0.943}\small{$\pm$}\small{0.021}&\small{0.955}\small{$\pm$}\small{0.019} \\ &&CrowdTCV& \small{0.925}\small{$\pm$}\small{0.027}&\small{0.946}\small{$\pm$}\small{0.026}&\small{0.968}\small{$\pm$}\small{0.015} \\ &&HTCV& \small{0.925}\small{$\pm$}\small{0.027}&\textbf{\small{0.952}}\small{$\pm$}\small{0.022}&\textbf{\small{0.974}}\small{$\pm$}\small{0.013} \\\midrule

\multirow{9}{*}{$\alpha = 0.4$}
&\multirow{3}{*}{0.25} &TCV& \small{0.691}\small{$\pm$}\small{0.073}&\small{0.790}\small{$\pm$}\small{0.047}&\small{0.809}\small{$\pm$}\small{0.048} \\ &&CrowdTCV& \small{0.804}\small{$\pm$}\small{0.050}&\small{0.901}\small{$\pm$}\small{0.028}&\textbf{\small{0.946}}\small{$\pm$}\small{0.022} \\ &&HTCV& \textbf{\small{0.808}}\small{$\pm$}\small{0.049}&\textbf{\small{0.904}}\small{$\pm$}\small{0.028}&\small{0.945}\small{$\pm$}\small{0.022} \\\cmidrule(lr){2-6}
&\multirow{3}{*}{1.0} &TCV& \small{0.821}\small{$\pm$}\small{0.047}&\small{0.859}\small{$\pm$}\small{0.036}&\small{0.875}\small{$\pm$}\small{0.036} \\ &&CrowdTCV& \small{0.832}\small{$\pm$}\small{0.044}&\small{0.900}\small{$\pm$}\small{0.035}&\small{0.946}\small{$\pm$}\small{0.020} \\ &&HTCV& \textbf{\small{0.836}}\small{$\pm$}\small{0.043}&\textbf{\small{0.904}}\small{$\pm$}\small{0.032}&\textbf{\small{0.947}}\small{$\pm$}\small{0.020} \\\cmidrule(lr){2-6}
&\multirow{3}{*}{2.5} &TCV& \textbf{\small{0.901}}\small{$\pm$}\small{0.027}&\small{0.921}\small{$\pm$}\small{0.029}&\small{0.935}\small{$\pm$}\small{0.026} \\ &&CrowdTCV& \small{0.895}\small{$\pm$}\small{0.031}&\small{0.923}\small{$\pm$}\small{0.028}&\small{0.950}\small{$\pm$}\small{0.019} \\ &&HTCV& \small{0.895}\small{$\pm$}\small{0.030}&\textbf{\small{0.926}}\small{$\pm$}\small{0.025}&\textbf{\small{0.957}}\small{$\pm$}\small{0.018} \\\midrule

\multirow{9}{*}{$\alpha = 0.2$}
&\multirow{3}{*}{0.25} &TCV& \small{0.599}\small{$\pm$}\small{0.088}&\small{0.688}\small{$\pm$}\small{0.077}&\small{0.738}\small{$\pm$}\small{0.060} \\ &&CrowdTCV& \small{0.689}\small{$\pm$}\small{0.080}&\small{0.826}\small{$\pm$}\small{0.046}&\textbf{\small{0.899}}\small{$\pm$}\small{0.031} \\ &&HTCV& \textbf{\small{0.693}}\small{$\pm$}\small{0.082}&\textbf{\small{0.828}}\small{$\pm$}\small{0.049}&\small{0.898}\small{$\pm$}\small{0.034} \\\cmidrule(lr){2-6}
&\multirow{3}{*}{1.0} &TCV& \small{0.733}\small{$\pm$}\small{0.070}&\small{0.791}\small{$\pm$}\small{0.055}&\small{0.815}\small{$\pm$}\small{0.041} \\ &&CrowdTCV& \small{0.729}\small{$\pm$}\small{0.074}&\small{0.836}\small{$\pm$}\small{0.043}&\textbf{\small{0.904}}\small{$\pm$}\small{0.033} \\ &&HTCV& \textbf{\small{0.740}}\small{$\pm$}\small{0.072}&\textbf{\small{0.841}}\small{$\pm$}\small{0.038}&\small{0.901}\small{$\pm$}\small{0.031} \\\cmidrule(lr){2-6}
&\multirow{3}{*}{2.5} &TCV& \textbf{\small{0.856}}\small{$\pm$}\small{0.041}&\small{0.878}\small{$\pm$}\small{0.036}&\small{0.888}\small{$\pm$}\small{0.032} \\ &&CrowdTCV& \small{0.844}\small{$\pm$}\small{0.048}&\small{0.873}\small{$\pm$}\small{0.035}&\small{0.905}\small{$\pm$}\small{0.027} \\ &&HTCV& \small{0.848}\small{$\pm$}\small{0.041}&\textbf{\small{0.881}}\small{$\pm$}\small{0.036}&\textbf{\small{0.913}}\small{$\pm$}\small{0.026} \\
\bottomrule
    \end{tabular}
    \end{small}
\end{table}

When there are a portion of adversarial users as stated in setting (2), we consider adversarial users whose accuracy level $\gamma_u$ may take negative values as discussed above. 
The results for Gumbel and normal noises under setting (2) are shown in Table~\ref{tab:simulation_advers_extended_gumbel} and Table~\ref{tab:simulation_advers_extended_normal} respectively. It can be seen that in this case, the difference between the methods is even more pronounced.
\begin{table}[!htbp]
    \centering
    \caption{Kendall's tau correlation for different methods under noise from the Gumbel distribution when a third of the users are \emph{adversarial}. The \textbf{bold} number highlights the highest performance and the \underline{underlined} number indicates a tie.
    \label{tab:simulation_advers_extended_gumbel}}
    \vspace{0.2in}
    \begin{small}
    \begin{tabular}{cccccc}
      \toprule
      \multirow{2}{*}{\shortstack[t]{Observ.\\Ratio}}&\multirow{2}{*}{$\gamma_B$}&\multirow{2}{*}{Methods}&\multicolumn{3}{c}{$\gamma_A$}\\
      \cmidrule(lr){4-6}
      &&& 2.5 & 5 & 10  \\
      \midrule
\multirow{9}{*}{$\alpha = 0.8$}
&\multirow{3}{*}{0.25} &BTL& \small{0.443}\small{$\pm$}\small{0.107}&\small{0.569}\small{$\pm$}\small{0.096}&\small{0.614}\small{$\pm$}\small{0.085} \\ &&CrowdBT& \underline{\small{0.852}}\small{$\pm$}\small{0.044}&\small{0.925}\small{$\pm$}\small{0.023}&\textbf{\small{0.967}}\small{$\pm$}\small{0.017} \\ &&HBTL& \underline{\small{0.852}}\small{$\pm$}\small{0.045}&\textbf{\small{0.926}}\small{$\pm$}\small{0.023}&\small{0.966}\small{$\pm$}\small{0.017} \\\cmidrule(lr){2-6}
&\multirow{3}{*}{1.0} &BTL& \small{0.575}\small{$\pm$}\small{0.089}&\small{0.663}\small{$\pm$}\small{0.071}&\small{0.710}\small{$\pm$}\small{0.074} \\ &&CrowdBT& \small{0.873}\small{$\pm$}\small{0.037}&\small{0.931}\small{$\pm$}\small{0.023}&\textbf{\small{0.967}}\small{$\pm$}\small{0.014} \\ &&HBTL& \textbf{\small{0.875}}\small{$\pm$}\small{0.037}&\textbf{\small{0.932}}\small{$\pm$}\small{0.024}&\small{0.966}\small{$\pm$}\small{0.017} \\\cmidrule(lr){2-6}
&\multirow{3}{*}{2.5} &BTL& \small{0.725}\small{$\pm$}\small{0.057}&\small{0.780}\small{$\pm$}\small{0.046}&\small{0.798}\small{$\pm$}\small{0.047} \\ &&CrowdBT& \underline{\small{0.931}}\small{$\pm$}\small{0.025}&\small{0.948}\small{$\pm$}\small{0.019}&\small{0.966}\small{$\pm$}\small{0.016} \\ &&HBTL& \underline{\small{0.931}}\small{$\pm$}\small{0.025}&\textbf{\small{0.951}}\small{$\pm$}\small{0.019}&\textbf{\small{0.973}}\small{$\pm$}\small{0.015} \\\midrule

\multirow{9}{*}{$\alpha = 0.6$}
&\multirow{3}{*}{0.25} &BTL& \small{0.384}\small{$\pm$}\small{0.122}&\small{0.491}\small{$\pm$}\small{0.107}&\small{0.557}\small{$\pm$}\small{0.095} \\ &&CrowdBT& \small{0.822}\small{$\pm$}\small{0.046}&\small{0.908}\small{$\pm$}\small{0.030}&\small{0.953}\small{$\pm$}\small{0.019} \\ &&HBTL& \textbf{\small{0.824}}\small{$\pm$}\small{0.044}&\textbf{\small{0.910}}\small{$\pm$}\small{0.028}&\textbf{\small{0.954}}\small{$\pm$}\small{0.018} \\\cmidrule(lr){2-6}
&\multirow{3}{*}{1.0} &BTL& \small{0.546}\small{$\pm$}\small{0.097}&\small{0.627}\small{$\pm$}\small{0.078}&\small{0.670}\small{$\pm$}\small{0.080} \\ &&CrowdBT& \small{0.852}\small{$\pm$}\small{0.037}&\small{0.911}\small{$\pm$}\small{0.029}&\small{0.954}\small{$\pm$}\small{0.018} \\ &&HBTL& \textbf{\small{0.854}}\small{$\pm$}\small{0.037}&\textbf{\small{0.914}}\small{$\pm$}\small{0.028}&\textbf{\small{0.956}}\small{$\pm$}\small{0.019} \\\cmidrule(lr){2-6}
&\multirow{3}{*}{2.5} &BTL& \small{0.684}\small{$\pm$}\small{0.078}&\small{0.736}\small{$\pm$}\small{0.064}&\small{0.755}\small{$\pm$}\small{0.062} \\ &&CrowdBT& \small{0.910}\small{$\pm$}\small{0.028}&\small{0.934}\small{$\pm$}\small{0.025}&\small{0.960}\small{$\pm$}\small{0.016} \\ &&HBTL& \textbf{\small{0.912}}\small{$\pm$}\small{0.029}&\textbf{\small{0.936}}\small{$\pm$}\small{0.024}&\textbf{\small{0.965}}\small{$\pm$}\small{0.017} \\\midrule

\multirow{9}{*}{$\alpha = 0.4$}
&\multirow{3}{*}{0.25} &BTL& \small{0.323}\small{$\pm$}\small{0.130}&\small{0.405}\small{$\pm$}\small{0.132}&\small{0.485}\small{$\pm$}\small{0.109} \\ &&CrowdBT& \small{0.742}\small{$\pm$}\small{0.169}&\underline{\small{0.877}}\small{$\pm$}\small{0.033}&\textbf{\small{0.934}}\small{$\pm$}\small{0.025} \\ &&HBTL& \textbf{\small{0.766}}\small{$\pm$}\small{0.059}&\underline{\small{0.877}}\small{$\pm$}\small{0.035}&\small{0.933}\small{$\pm$}\small{0.024} \\\cmidrule(lr){2-6}
&\multirow{3}{*}{1.0} &BTL& \small{0.448}\small{$\pm$}\small{0.118}&\small{0.544}\small{$\pm$}\small{0.096}&\small{0.583}\small{$\pm$}\small{0.094} \\ &&CrowdBT& \small{0.810}\small{$\pm$}\small{0.044}&\small{0.886}\small{$\pm$}\small{0.031}&\underline{\small{0.934}}\small{$\pm$}\small{0.026} \\ &&HBTL& \textbf{\small{0.819}}\small{$\pm$}\small{0.045}&\textbf{\small{0.891}}\small{$\pm$}\small{0.031}&\underline{\small{0.934}}\small{$\pm$}\small{0.029} \\\cmidrule(lr){2-6}
&\multirow{3}{*}{2.5} &BTL& \small{0.627}\small{$\pm$}\small{0.087}&\small{0.660}\small{$\pm$}\small{0.075}&\small{0.698}\small{$\pm$}\small{0.063} \\ &&CrowdBT& \small{0.879}\small{$\pm$}\small{0.034}&\small{0.913}\small{$\pm$}\small{0.027}&\small{0.939}\small{$\pm$}\small{0.023} \\ &&HBTL& \textbf{\small{0.880}}\small{$\pm$}\small{0.032}&\textbf{\small{0.914}}\small{$\pm$}\small{0.029}&\textbf{\small{0.948}}\small{$\pm$}\small{0.022} \\\midrule

\multirow{9}{*}{$\alpha = 0.2$}
&\multirow{3}{*}{0.25} &BTL& \small{0.246}\small{$\pm$}\small{0.145}&\small{0.305}\small{$\pm$}\small{0.151}&\small{0.361}\small{$\pm$}\small{0.143} \\ &&CrowdBT& \small{0.613}\small{$\pm$}\small{0.235}&\textbf{\small{0.712}}\small{$\pm$}\small{0.356}&\underline{\small{0.848}}\small{$\pm$}\small{0.256} \\ &&HBTL& \textbf{\small{0.614}}\small{$\pm$}\small{0.263}&\small{0.709}\small{$\pm$}\small{0.380}&\underline{\small{0.848}}\small{$\pm$}\small{0.249} \\\cmidrule(lr){2-6}
&\multirow{3}{*}{1.0} &BTL& \small{0.336}\small{$\pm$}\small{0.154}&\small{0.407}\small{$\pm$}\small{0.127}&\small{0.452}\small{$\pm$}\small{0.132} \\ &&CrowdBT& \small{0.644}\small{$\pm$}\small{0.282}&\small{0.795}\small{$\pm$}\small{0.176}&\small{0.878}\small{$\pm$}\small{0.038} \\ &&HBTL& \textbf{\small{0.650}}\small{$\pm$}\small{0.281}&\textbf{\small{0.807}}\small{$\pm$}\small{0.172}&\textbf{\small{0.888}}\small{$\pm$}\small{0.040} \\\cmidrule(lr){2-6}
&\multirow{3}{*}{2.5} &BTL& \small{0.498}\small{$\pm$}\small{0.106}&\small{0.548}\small{$\pm$}\small{0.103}&\small{0.571}\small{$\pm$}\small{0.098} \\ &&CrowdBT& \small{0.803}\small{$\pm$}\small{0.049}&\small{0.858}\small{$\pm$}\small{0.039}&\small{0.897}\small{$\pm$}\small{0.032} \\ &&HBTL& \textbf{\small{0.807}}\small{$\pm$}\small{0.049}&\textbf{\small{0.865}}\small{$\pm$}\small{0.039}&\textbf{\small{0.900}}\small{$\pm$}\small{0.029} \\
\bottomrule
    \end{tabular}
    \end{small}
\end{table}

\begin{table}[!ht]
    \centering
    \caption{Kendall tau correlation for different methods under noise from the normal distribution when a third of the users are \textit{adversarial}. The \textbf{bold} number highlights the highest performance and the \underline{underlined} number indicates a tie.
    }
    \label{tab:simulation_advers_extended_normal}
    \vspace{0.2in}

    \begin{small}
    \begin{tabular}{cccccc}
      \toprule
      \multirow{2}{*}{\shortstack[t]{Observ.\\Ratio}}&\multirow{2}{*}{$\gamma_B$}&\multirow{2}{*}{Methods}&\multicolumn{3}{c}{$\gamma_A$}\\
      \cmidrule(lr){4-6}
      &&& 2.5 & 5 & 10  \\
      \midrule
\multirow{9}{*}{$\alpha = 0.8$}
&\multirow{3}{*}{0.25} &TCV& \small{0.471}\small{$\pm$}\small{0.105}&\small{0.590}\small{$\pm$}\small{0.095}&\small{0.640}\small{$\pm$}\small{0.075} \\ &&CrowdTCV& \underline{\small{0.882}}\small{$\pm$}\small{0.034}&\textbf{\small{0.938}}\small{$\pm$}\small{0.023}&\small{0.972}\small{$\pm$}\small{0.017} \\ &&HTCV& \underline{\small{0.882}}\small{$\pm$}\small{0.033}&\small{0.937}\small{$\pm$}\small{0.023}&\textbf{\small{0.973}}\small{$\pm$}\small{0.016} \\\cmidrule(lr){2-6}
&\multirow{3}{*}{1.0} &TCV& \small{0.642}\small{$\pm$}\small{0.083}&\small{0.694}\small{$\pm$}\small{0.068}&\small{0.722}\small{$\pm$}\small{0.064} \\ &&CrowdTCV& \small{0.893}\small{$\pm$}\small{0.030}&\small{0.945}\small{$\pm$}\small{0.020}&\small{0.973}\small{$\pm$}\small{0.016} \\ &&HTCV& \textbf{\small{0.895}}\small{$\pm$}\small{0.031}&\textbf{\small{0.947}}\small{$\pm$}\small{0.019}&\textbf{\small{0.975}}\small{$\pm$}\small{0.017} \\\cmidrule(lr){2-6}
&\multirow{3}{*}{2.5} &TCV& \small{0.772}\small{$\pm$}\small{0.055}&\small{0.804}\small{$\pm$}\small{0.045}&\small{0.821}\small{$\pm$}\small{0.050} \\ &&CrowdTCV& \textbf{\small{0.945}}\small{$\pm$}\small{0.021}&\small{0.956}\small{$\pm$}\small{0.019}&\small{0.978}\small{$\pm$}\small{0.014} \\ &&HTCV& \small{0.944}\small{$\pm$}\small{0.021}&\textbf{\small{0.960}}\small{$\pm$}\small{0.019}&\textbf{\small{0.982}}\small{$\pm$}\small{0.013} \\\midrule

\multirow{9}{*}{$\alpha = 0.6$}
&\multirow{3}{*}{0.25} &TCV& \small{0.416}\small{$\pm$}\small{0.129}&\small{0.527}\small{$\pm$}\small{0.107}&\small{0.552}\small{$\pm$}\small{0.099} \\ &&CrowdTCV& \underline{\small{0.847}}\small{$\pm$}\small{0.039}&\small{0.924}\small{$\pm$}\small{0.025}&\underline{\small{0.960}}\small{$\pm$}\small{0.020} \\ &&HTCV& \underline{\small{0.847}}\small{$\pm$}\small{0.039}&\textbf{\small{0.925}}\small{$\pm$}\small{0.023}&\underline{\small{0.960}}\small{$\pm$}\small{0.020} \\\cmidrule(lr){2-6}
&\multirow{3}{*}{1.0} &TCV& \small{0.569}\small{$\pm$}\small{0.086}&\small{0.648}\small{$\pm$}\small{0.066}&\small{0.686}\small{$\pm$}\small{0.080} \\ &&CrowdTCV& \small{0.866}\small{$\pm$}\small{0.036}&\small{0.930}\small{$\pm$}\small{0.024}&\underline{\small{0.966}}\small{$\pm$}\small{0.018} \\ &&HTCV& \textbf{\small{0.870}}\small{$\pm$}\small{0.036}&\textbf{\small{0.932}}\small{$\pm$}\small{0.025}&\underline{\small{0.966}}\small{$\pm$}\small{0.018} \\\cmidrule(lr){2-6}
&\multirow{3}{*}{2.5} &TCV& \small{0.718}\small{$\pm$}\small{0.060}&\small{0.762}\small{$\pm$}\small{0.045}&\small{0.786}\small{$\pm$}\small{0.055} \\ &&CrowdTCV& \textbf{\small{0.926}}\small{$\pm$}\small{0.027}&\small{0.949}\small{$\pm$}\small{0.023}&\small{0.969}\small{$\pm$}\small{0.014} \\ &&HTCV& \small{0.925}\small{$\pm$}\small{0.027}&\textbf{\small{0.952}}\small{$\pm$}\small{0.020}&\textbf{\small{0.972}}\small{$\pm$}\small{0.014} \\\midrule

\multirow{9}{*}{$\alpha = 0.4$}
&\multirow{3}{*}{0.25} &TCV& \small{0.359}\small{$\pm$}\small{0.119}&\small{0.472}\small{$\pm$}\small{0.116}&\small{0.514}\small{$\pm$}\small{0.103} \\ &&CrowdTCV& \small{0.797}\small{$\pm$}\small{0.053}&\small{0.893}\small{$\pm$}\small{0.034}&\textbf{\small{0.942}}\small{$\pm$}\small{0.022} \\ &&HTCV& \textbf{\small{0.799}}\small{$\pm$}\small{0.048}&\textbf{\small{0.896}}\small{$\pm$}\small{0.031}&\small{0.938}\small{$\pm$}\small{0.022} \\\cmidrule(lr){2-6}
&\multirow{3}{*}{1.0} &TCV& \small{0.487}\small{$\pm$}\small{0.116}&\small{0.577}\small{$\pm$}\small{0.088}&\small{0.587}\small{$\pm$}\small{0.088} \\ &&CrowdTCV& \small{0.842}\small{$\pm$}\small{0.049}&\small{0.898}\small{$\pm$}\small{0.029}&\textbf{\small{0.945}}\small{$\pm$}\small{0.021} \\ &&HTCV& \textbf{\small{0.843}}\small{$\pm$}\small{0.046}&\textbf{\small{0.902}}\small{$\pm$}\small{0.027}&\small{0.944}\small{$\pm$}\small{0.022} \\\cmidrule(lr){2-6}
&\multirow{3}{*}{2.5} &TCV& \small{0.648}\small{$\pm$}\small{0.073}&\small{0.704}\small{$\pm$}\small{0.071}&\small{0.718}\small{$\pm$}\small{0.066} \\ &&CrowdTCV& \underline{\small{0.895}}\small{$\pm$}\small{0.031}&\small{0.925}\small{$\pm$}\small{0.031}&\small{0.951}\small{$\pm$}\small{0.021} \\ &&HTCV& \underline{\small{0.895}}\small{$\pm$}\small{0.030}&\textbf{\small{0.929}}\small{$\pm$}\small{0.028}&\textbf{\small{0.957}}\small{$\pm$}\small{0.018} \\\midrule

\multirow{9}{*}{$\alpha = 0.2$}
&\multirow{3}{*}{0.25} &TCV& \small{0.259}\small{$\pm$}\small{0.147}&\small{0.349}\small{$\pm$}\small{0.135}&\small{0.382}\small{$\pm$}\small{0.133} \\ &&CrowdTCV& \small{0.600}\small{$\pm$}\small{0.340}&\small{0.826}\small{$\pm$}\small{0.044}&\textbf{\small{0.895}}\small{$\pm$}\small{0.038} \\ &&HTCV& \textbf{\small{0.636}}\small{$\pm$}\small{0.282}&\textbf{\small{0.828}}\small{$\pm$}\small{0.044}&\small{0.893}\small{$\pm$}\small{0.036} \\\cmidrule(lr){2-6}
&\multirow{3}{*}{1.0} &TCV& \small{0.397}\small{$\pm$}\small{0.119}&\small{0.436}\small{$\pm$}\small{0.115}&\small{0.469}\small{$\pm$}\small{0.100} \\ &&CrowdTCV& \small{0.721}\small{$\pm$}\small{0.065}&\textbf{\small{0.834}}\small{$\pm$}\small{0.043}&\small{0.901}\small{$\pm$}\small{0.033} \\ &&HTCV& \textbf{\small{0.736}}\small{$\pm$}\small{0.066}&\small{0.832}\small{$\pm$}\small{0.046}&\textbf{\small{0.905}}\small{$\pm$}\small{0.032} \\\cmidrule(lr){2-6}
&\multirow{3}{*}{2.5} &TCV& \small{0.518}\small{$\pm$}\small{0.102}&\small{0.577}\small{$\pm$}\small{0.098}&\small{0.600}\small{$\pm$}\small{0.077} \\ &&CrowdTCV& \small{0.843}\small{$\pm$}\small{0.049}&\small{0.873}\small{$\pm$}\small{0.037}&\small{0.908}\small{$\pm$}\small{0.030} \\ &&HTCV& \textbf{\small{0.848}}\small{$\pm$}\small{0.041}&\textbf{\small{0.880}}\small{$\pm$}\small{0.036}&\textbf{\small{0.917}}\small{$\pm$}\small{0.028} \\
\bottomrule
    \end{tabular}
    \end{small}
\end{table}

\subsection{Experimental Results on Real-World Data}
We evaluate our method on two real-world datasets. The first one named ``Reading Level'' ~\citep{chen2013pairwise} contains English text excerpts whose reading difficulty level is compared by workers. $624$ workers annotated $490$ excerpts which resulting in a total of $12,728$ pairwise comparisons. 
We also used Mechanical Turk to collect another dataset named ``Country Population''. In this crowdsourcing task, we asked workers to compare the population between two countries and pick the one which has more population. Since the population ranking of countries has a universal consensus, which can be obtained by looking up demographic data, it is a better choice than those movie rankings which subjects to personal preferences. There were 15 countries as shown in Table \ref{tab:cp_gt} which made up to 105 pairwise comparisons. The values were collected according to the latest demography statistics on Wikipedia for each country as of March 2019. Each user was asked 16 pairs randomly selected from all those 105 pairs. A total of 199 workers provided response to this task through Mechanical Turk. 
These two datasets were both collected in online crowdsourcing environments so that we can expect varying worker accuracy where effectiveness of our approach can be demonstrated. 

In real-world datasets, it may happen that two items from two subsets are never compared with each other, directly or indirectly. In such cases, the ranking will not be unique. Furthermore, if data is sparse, the estimates may suffer from overfitting. To address these issues, regularization is often used. While this can be done in a variety of ways, for the sake of comparison with CrowdBT, we use virtual node regularization~\citep{chen2013pairwise}. Specifically, it is assumed that there is a virtual item of utility $s_0 = 0$ which is compared to all other items by a virtual user. This leads to the loss function $\cL + \lambda_0\cL_0$, where $\cL_0 = - \sum_{i\in [n] }\log  F\left(s_0-s_{i}\right)-\sum_{i\in [n] }\log F\left(s_{i}-s_0\right)$ and $\lambda_0 \geq 0$ is a tuning parameter.

We evaluate the performance of the methods for $\lambda_0=0,1,5,10$. For different values of $\lambda_0$, HBTL performs best more often than any other method and, in particular, it performs best for $\lambda_0=0$. Table~\ref{tab:real_data_exp} reports the best performance of each method across different regularization values for the two real-world data experiment. It can be observed that HBTL and HTCV outperform their counterparts, CrowdBT and CrowdTCV, as well as the uniform models, BTL and TCV. 

\begin{table}[h]
        \begin{center}
        \caption{Ground truth for ``Country Population'' dataset.}
        \label{tab:cp_gt}
        \vspace{0.1in}
        \begin{tabular}{cc}
        \toprule
         Country & Population (million) \\
        \midrule
        China &	1410 \\
        India &	1340 \\
        United States &	324 \\
        Indonesia &	264 \\
        Brazil &	209 \\
        Pakistan &	197 \\
        Nigeria & 	191 \\
        Bangladesh &	165 \\
        Russia &	144 \\
        Mexico & 	129 \\
        Japan & 	127 \\
        Ethiopia &	105 \\
        Philippines	& 104.9 \\
        Egypt &	97.6 \\
        Vietnam &	95.5 \\
            \bottomrule
        \end{tabular}
        \end{center}
\end{table}

\begin{table*}[tb]
    \begin{small}
        \begin{center}
        \caption{Performance of ranking algorithms on real-world dataset. The \textbf{bold} number highlights the highest performance.\label{tab:real_data_exp}  }
            \vspace{0.1in}
        \begin{tabular}{ccccccc}
        \toprule
        Dataset                  & BTL    & TCV & CrowdBT & CrowdTCV & HBTL & HTCV \\
        \midrule
        Reading Level      & 0.3472 & 0.3452 & 0.3737 & 0.3672 & \textbf{0.3763} & 0.3729 \\
        Country Population & 0.7524 & 0.7524 & 0.7714 & 0.7714 & \textbf{0.7905} & 0.7714 \\
            \bottomrule
        \end{tabular}
        \end{center}
         \end{small}
    \end{table*}



\subsection{Analysis on regularization effects}

Detailed result with various regularization settings can be found in Table \ref{tab:real_data_reg_rd} and Table \ref{tab:real_data_reg_cp}. The reported values are Kendall's tau correlation. It shows that without regularization our method outperforms other methods. And with virtual node trick, it shows relative amount of improvement in the final ranking result, yet not essential. However, this method needs to tune another parameter $\lambda_0$. If no gold/ground-truth comparison is given, there will be no validation standard to tune this parameter.
Furthermore, the performance of the proposed methods is less dependent on the regularization parameter, which facilitates their application to real data. It is also interesting to see that our method is less prone to be affected by the regularization parameter.

\begin{table}[tb]
        \begin{center}
        \caption{Performance of ranking algorithms for the ``Reading Level'' dataset with different regularization parameters. The \textbf{bold} number highlights the highest performance.\label{tab:real_data_reg_rd}  }
        \vspace{0.1in}
        \begin{tabular}{ccccc}
        \toprule
         & $\lambda_0 = 0$ &$\lambda_0 = 1$ & $\lambda_0=5$ & $\lambda_0=10$ \\
        \midrule
            BTL      &         0.3299 &         0.3433 &         0.3472 & 0.3402 \\
            TCV      &         0.3294 &         0.3423 &         0.3452 & 0.3375 \\
            CrowdBT  &         0.3490 & \textbf{0.3737}&         0.3648 & 0.3535 \\
            CrowdTCV &         0.3512 &        {0.3672}&         0.3511 & 0.3388 \\
            HBTL     & \textbf{0.3608}&         0.3660 &         0.3719 & \textbf{0.3763} \\
            HTCV     &         0.3578 &         0.3696 & \textbf{0.3729}& {0.3680} \\
            \bottomrule
        \end{tabular}
        \end{center}
\end{table}

\begin{table}[tb]
        \begin{center}
        \caption{Performance of ranking algorithms for the ``Country Population'' dataset with different regularization parameters. The \textbf{bold} number highlights the highest performance.}
        \label{tab:real_data_reg_cp}
        \vspace{0.1in}        \begin{tabular}{ccccc}
        \toprule
         & $\lambda_0 = 0$ &$\lambda_0 = 1$ & $\lambda_0=5$ & $\lambda_0=10$ \\
        \midrule
            BTL & 0.7524 & 0.7524 & 0.7524 & {0.7524} \\
            TCV & 0.7524 & 0.7524 & 0.7524 & {0.7524} \\
            CrowdBT & 0.7714 & 0.7714 & \textbf{0.7714} & {0.7524} \\
            CrowdTCV & 0.7714 & 0.7714 & \textbf{0.7714} & {0.7524} \\
            HBTL & \textbf{0.7905} & \textbf{0.7905} & 0.7524 &{0.7524} \\
            HTCV & 0.7714 & 0.7714 & {0.7524} & {0.7524} \\
            \bottomrule
        \end{tabular}
        \end{center}
\end{table}

\section{Conclusions and Future Work}\label{sec:conc}
In this paper, we propose the heterogeneous Thurstone model for pairwise comparisons and partial rankings when data is produced by a population of users with diverse levels of expertise, as is often the case in real-world applications. The proposed model maintains the generality of Thurstone's framework and thus also extends common models such as Bradley-Terry-Luce, Thurstone's Case V, and Plackett-Luce. We also developed an alternating gradient descent algorithm to estimate the score vector and expertise level vector simultaneously. We prove the local linear convergence of our algorithm for general HTM models satisfying mild conditions. We also prove the convergence of our algorithm for the two most common noise distributions, which leads to the HBTL and HTCV models. Experiments on both synthetic and real data show that our proposed model and algorithm generally outperforms the competing methods, sometimes by a significant margin.

There are several interesting future directions that could be explored. First, it would be of great importance to devise a provable initialization algorithm since our current analysis relies on certain initialization methods that are guaranteed to be close to the true values. Another direction is extending the algorithm and analysis to the case of partial ranking such as the Plackett-Luce model. Finally, lower bounds on the estimation error would enable better evaluating algorithms for rank aggregation in heterogeneous Thurstone models.

\appendix
\section{Implications of Specific Models}\label{sec:analysis-models}
Our Theorem \ref{thm:convergence_general} is for general HTM models that satisfy Conditions \ref{assump:strong_convex}, \ref{assump:smooth}, \ref{lemma:FOS} and \ref{lemma:staterr_general}. In this subsection, we will show that the linear convergence rate of Algorithm \ref{alg:PGD} can also be attained for specific models without assuming theses conditions when the random noise $\epsilon_i$ in \eqref{eq:latent} follows the Gumbel distribution and the Gaussian distribution respectively.

\subsection{Heterogeneous BTL model}
We first consider the model with Gumbel noise. Specifically, $\{\epsilon_i\}_{i=1,\ldots,n}$ follow the Gumbel distribution with mean $0$ and scale parameter $1$. Then we obtain the HBTL model defined in \eqref{eq:new_comp_Gumbel}. The following corollary states the convergence result of Algorithm \ref{alg:PGD} for HBTL models.
\begin{corollary}\label{coro:convergence_gumbel}
Consider the HBTL model in \eqref{eq:new_comp_Gumbel} and assume the sample size $k\geq n^2\log(mn)/m^2$. Let $\|\bs^*\|_{\infty}= \smax$, $\max_{u}|\gamma^{*u}|=\gammamax$ and $\min_{u}|\gamma^{*u}|=\gammamin$. Assume $\gammamax\smax= C_0$ for a constant $C_0\geq 1/2$ and
\begin{align*}
    \smax\leq \frac{\sqrt{m}\|\bs^*\|_2}{n}\cdot\frac{\gammamin e^{5C_0}}{32\sqrt{2}\gammamax(1+e^{5C_0})^2}.
\end{align*}
Suppose the initialization points $\bs^{(0)}$ and $\bgamma^{(0)}$ satisfy that $\|\bs^{(0)}-\bs^*\|_2^2+\|\bgamma^{(0)}-\bgamma^*\|_2^2\leq r^2$, where $r=\min\{\|\bs^*\|_2/2,\gammamin/2,\smax,\sqrt{\gammamax\smax}\}$. If we set the step size small enough such that
\begin{align*}
    \eta_1=\eta_2<\frac{mne^{5C_0}\Gamma_1^2}{6(1+e^{5C_0})^2(m\Gamma_2^4+32n^2C_0^2)},
\end{align*}
where $\Gamma_1=\min\{\gammamin/2,\|\bs^*\|_2\}$ and $\Gamma_2=\max\{2\gammamax,2\|\bs^*\|_2\}$, then the output of Algorithm \ref{alg:PGD} satisfies
\begin{align*}
    &\|\bs^{(T)}-\bs^*\|_2^2+\|\bgamma^{(T)}-\bgamma^*\|_2^2\leq r^2\rho^T+\frac{\Lambda n^2\log(4mn^2)}{mk}
\end{align*}
with probability at least $1-1/n$, where $\rho=1-\eta(\mu-6\eta(\Gamma_2^4/n^2+32C_0^2/m))/2$ and $\Lambda$ is a constant which only which depends on $C_0,\gammamax$ and $\Gamma_1$.
\end{corollary}

\begin{remark}
According to Corollary \ref{coro:convergence_gumbel}, when the initial points $\bs^{(0)}$ and $\bgamma^{(0)}$ lie in a small neighborhood of the unknown parameter $\bs^*,\bgamma^*$, the proposed algorithm converges linearly fast to a term in the order of $O(n^2\log(mn^2)/(mk))$, which is called the statistical error of the HBTL model. Note that when $m=1$, the statistical error reduces to $O(n^2\log(n)/k)$, which matches the state-of-the-art estimation error bound for single user BTL model \citep{negahban2017}. In addition, we assumed that $\|\bs^*\|_{\infty}\lesssim O(\sqrt{m}/n\|\bs^*\|_2)$ in order to derive the linear convergence of Algorithm \ref{alg:PGD}. When $m$ is in the same order of $n$, the requirement reduces to $\|\bs^*\|_{\infty}\lesssim O(\|\bs^*\|_2/\sqrt{n})$. This assumption is similar to the spikiness assumption in \cite{agarwal2012noisy,negahban2012restricted}, which ensures that there are not too many items that have zero or nearly zero scores.
\end{remark}

\subsection{Heterogeneous Thurstone Case V model}
Now we consider the HTM model with Gaussian noise. Assume that $\{\epsilon_i\}_{i=1,\ldots,n}$ are i.i.d. from $N(0,1)$. Then the general HTM model becomes HTCV model defined in \eqref{eq:new_comp_normal}, which generalizes the single user TCV model \citep{thurstone1927}. Before we present the convergence results of Algorithm \ref{alg:PGD} for this model, we first remark some notations of the normal distribution to simplify the presentation. In particular, let $\Phi(x)$ be the CDF of standard normal distribution. We define $H(x)=(\Phi'(x)^2-\Phi(x)\Phi''(x))/\Phi(x)^2$, which can be verified to be a monotonically decreasing function.
\begin{corollary}\label{coro:convergence_normal}
Consider the HTCV model in \eqref{eq:new_comp_normal} and assume the sample size $k\geq n^2\log(mn)/m^2$. $\smax$, $\gammamax$, $\gammamin$ and $C_0$ are defined the same as in Corollary \ref{coro:convergence_gumbel}. Assume $\smax$ satisfies
\begin{align*}
    \smax\leq\frac{\sqrt{m}\|\bs^*\|_2}{n}\cdot\frac{\gammamin H(5C_0)}{30\gammamax(\Phi(-5C_0)^{-1}+H(-5C_0))}.
\end{align*}
Suppose the initialization points $\bs^{(0)}$ and $\bgamma^{(0)}$ satisfy that $\|\bs^{(0)}-\bs^*\|_2^2+\|\bgamma^{(0)}-\bgamma^*\|_2^2\leq r^2$, where $r=\min\{\|\bs^*\|_2/2,\gammamin/2,\smax,\sqrt{\gammamax\smax}\}$. If we set the step size
\begin{align*}
    \eta_1=\eta_2<\frac{mn\Gamma_1^2H(5C_0)}{6(m\Gamma_2^4+50n^2C_0^2)H(-5C_0)^2},
\end{align*}
where $\Gamma_1=\min\{\gammamin/2,\|\bs^*\|_2\}$ and $\Gamma_2=\max\{2\gammamax,2\|\bs^*\|_2\}$, then the output of Algorithm \ref{alg:PGD} satisfies
\begin{align*}
    &\|\bs^{(T)}-\bs^*\|_2^2+\|\bgamma^{(T)}-\bgamma^*\|_2^2\leq r^2\rho^T+\frac{\Lambda' n^2\log(4mn^2)}{mk}
\end{align*}
with probability at least $1-1/n$, where $\rho=1-\eta(\mu-6\eta(\Gamma_2^4/n^2+32C_0^2/m))/2$ and $\Lambda'$ is a constant which only depends on $C_0,\gammamax$ and $\Gamma_1$.
\end{corollary}
\begin{remark}
Corollary \ref{coro:convergence_normal} suggests that under suitable initialization, Algorithm \ref{alg:PGD} enjoys a linear convergence rate when the random noise follows the standard normal distribution. The statistical error for the HTCV model is in the order of $O(n^2\log(mn^2)/(mk))$. We again need the `spikiness' assumption on the unknown score vector $\bs^*$ in order to ensure the algorithm to find the true parameter. The results are almost the same as those of the HBTL model presented in Corollary \ref{coro:convergence_gumbel} except that the constants in the HTCV model depends on the normal CDF $\Phi$ and its first and second derivatives.
\end{remark}

\section{Proof of the Generic Model}
In this section, we provide the proof of Theorem \ref{thm:convergence_general} for general heterogeneous Thurstone models.
\begin{proof}[Proof of Theorem \ref{thm:convergence_general}]
According to the update in Algorithm \ref{alg:PGD} and the fact that $\one^{\top}\bs^*=0$, we have
\begin{align*}
    \|\bs^{(t+1)}-\bs^*\|_2^2&=\|(\Ib-\one\one^{\top}/n)(\tilde\bs^{(t+1)}-\bs^*)\|_2^2\\
    &\leq\|\tilde\bs^{(t+1)}-\bs^*\|_2^2\\
    &=\|\bs^{(t)}-\bs^*\|_2^2+\eta_1^2\|\nabla_{\bs} \cL(\bs^{(t)},\bgamma^{(t)})\|_2^2
    -2\eta_1\la\gradsb\cL(\bs^{(t)},\bgamma^{(t)}),\bs^{(t)}-\bs^*\ra,
\end{align*}
where the inequality comes from the fact that $\|\Ib-\one\one^{\top}/n\|_2\leq 1$. We first bound the second term on the right hand side above
\begin{align*}
    \|\nabla_{\bs} \cL(\bs^{(t)},\bgamma^{(t)})\|_2^2&\leq3\|\nabla_{\bs} \cL(\bs^{(t)},\bgamma^{(t)})-\nabla_{\bs} \cL(\bs^{(t)},\bgamma^*)\|_2^2+3\|\nabla_{\bs} \cL(\bs^{(t)},\bgamma^*)-\gradsb\cL(\bs^*,\bgamma^*)\|_2^2\\
    &\qquad+3\|\nabla_{\bs} \cL(\bs^*,\bgamma^*)-\gradsb\bar\cL(\bs^*,\bgamma^*)\|_2^2\\
    &\leq 3M_1^2\|\bgamma^{(t)}-\bgamma^*\|_2^2+3L_1^2\|\bs^{(t)}-\bs^*\|_2^2+3\epsilon_1(k,n)^2,
\end{align*}
where the first inequality is due to $\gradsb\bar\cL(\bs^*,\bgamma^*)=\zero$ and the second inequality is due to Conditions  \ref{assump:smooth}, \ref{lemma:FOS}, and \ref{lemma:staterr_general}. Now we bound the inner product term. Note that
\begin{align*}
    &\la\gradsb\cL(\bs^{(t)},\bgamma^{(t)}),\bs^{(t)}-\bs^*\ra\\
    &=\la\gradsb\cL(\bs^{(t)},\bgamma^{(t)})-\gradsb\cL(\bs^*,\bgamma^{(t)}),\bs^{(t)}-\bs^*\ra+\la\gradsb\cL(\bs^*,\bgamma^{(t)})-\gradsb\cL(\bs^*,\bgamma^*),\bs^{(t)}-\bs^*\ra\\
    &\qquad+\la\gradsb\cL(\bs^*,\bgamma^*)-\gradsb\bar\cL(\bs^*,\bgamma^*),\bs^{(t)}-\bs^*\ra.
\end{align*}
By strong convexity (Condition \ref{assump:strong_convex}) of $\cL$ we have
\begin{align}\label{eq:general_proof_innerproduct1}
    &\la\gradsb\cL(\bs^{(t)},\bgamma^{(t)})-\gradsb\cL(\bs^*,\bgamma^{(t)}),\bs^{(t)}-\bs^*\ra
    \geq\mu_1\|\bs^{(t)}-\bs^*\|_2^2.
\end{align}
Applying Young's inequality and Condition \ref{lemma:FOS}, we obtain
\begin{align}\label{eq:general_proof_innerproduct2}
    \lvert \la\gradsb\cL(\bs^*,\bgamma^{(t)})-\gradsb\cL(\bs^*,\bgamma^*),\bs^{(t)}-\bs^*\ra \rvert
    &\leq\|\gradsb\cL(\bs^*,\bgamma^{(t)})-\gradsb\cL(\bs^*,\bgamma^*)\|_2\cdot\|\bs^{(t)}-\bs^*\|_2\notag\\
    &\leq\frac{\alpha M_1^2}{2}\|\bgamma^{(t)}-\bgamma^*\|_2^2+\frac{1}{2\alpha}\|\bs^{(t)}-\bs^*\|_2^2,
\end{align}
where $\alpha>0$ is an arbitrarily chosen constant.
In addition, by Condition \ref{lemma:staterr_general} and Young's inequality we have
\begin{align}\label{eq:general_proof_innerproduct3}
    \lvert \la\gradsb\cL(\bs^*,\bgamma^*)-\gradsb\bar\cL(\bs^*,\bgamma^*),\bs^{(t)}-\bs^*\ra \rvert
    &\leq\|\gradsb\cL(\bs^*,\bgamma^*)-\gradsb\bar\cL(\bs^*,\bgamma^*)\|_2\cdot\|\bs^{(t)}-\bs^*\|_2\notag\\
    &\leq\frac{1}{2\mu_1}\epsilon_1(k,n)^2+\frac{\mu_1}{2}\|\bs^{(t)}-\bs^*\|_2^2.
\end{align}
Combining \eqref{eq:general_proof_innerproduct1}, \eqref{eq:general_proof_innerproduct2} and \eqref{eq:general_proof_innerproduct3}, we have
\begin{align*}
    \la\gradsb\cL(\bs^{(t)},\bgamma^{(t)}),\bs^{(t)}-\bs^*\ra
    &\geq\frac{\mu_1\alpha-1}{2\alpha}\|\bs^{(t)}-\bs^*\|_2^2-\frac{\alpha M_1^2}{2}\|\bgamma^{(t)}-\bgamma^*\|_2^2-\frac{1}{2\mu_1}\epsilon_1(k,n)^2.
\end{align*}
Therefore, we have
\begin{align}
    \|\bs^{(t+1)}-\bs^*\|_2^2&\leq\bigg(1+3L_1^2\eta_1^2-\eta_1\bigg(\mu_1-\frac{1}{\alpha}\bigg)\bigg)\|\bs^{(t)}-\bs^*\|_2^2+M_1^2(3\eta_1^2+\alpha\eta_1)\|\bgamma^{(t)}-\bgamma^*\|_2^2\notag\\
    &\qquad+(3\eta_1^2+\eta_1/\mu_1)\epsilon_1(k,n)^2.
\end{align}
Similarly, we can bound $\|\bgamma^{(t+1)}-\bgamma^*\|_2^2$ as follows
\begin{align}
    \|\bgamma^{(t+1)}-\bgamma^*\|_2^2&\leq\bigg(1+3L_2^2\eta_2^2-\eta_2\bigg(\mu_2-\frac{1}{\beta}\bigg)\bigg)\|\bgamma^{(t)}-\bgamma^*\|_2^2+M_2^2(3\eta_2^2+\beta\eta_2)\|\bs^{(t)}-\bs^*\|_2^2\notag\\
    &\qquad+(3\eta_2^2+\eta_2/\mu_2)\epsilon_2(k,n)^2,
\end{align}
where $\beta>0$ are arbitrarily chosen constants. In particular, set $\alpha=\mu_2/(4M_1^2)$, $\beta=\mu_1/(4M_2^2)$ and $\eta_1=\eta_2=\eta$. When $M_1,M_2\leq\sqrt{\mu_1\mu_2}/4$, we have
\begin{align}\label{eq:bs_contract_general}
    \|\bs^{(t+1)}-\bs^*\|_2^2+\|\bgamma^{(t+1)}-\bgamma^*\|_2^2
    &\leq(1+3(L_1^2+M_2^2)\eta^2-\mu_1\eta/2)\|\bs^{(t)}-\bs^*\|_2^2\notag\\
    &\qquad+(1+3(L_2^2+M_1^2)\eta_2^2-\mu_2\eta/2)\|\bgamma^{(t)}-\bgamma^*\|_2^2\notag\\
    &\qquad+(3\eta^2+\eta/\mu_1)\epsilon_1(k,n)^2+(3\eta^2+\eta/\mu_2)\epsilon_2(k,n)^2\notag\\
    &\leq(1+3(L^2+M^2)\eta^2-\mu\eta/2)(\|\bs^{(t)}-\bs^*\|_2^2+\|\bgamma^{(t)}-\bgamma^*\|_2^2)\notag\\
    &\qquad+(3\eta^2+\eta/\mu)(\epsilon_1(k,n)^2+\epsilon_2(k,n)^2),
\end{align}
where $L=\max\{L_1,L_2\}$, $M=\max\{M_1,M_2\}$ and $\mu=\min\{\mu_1,\mu_2\}$. Note that we have $\|\bs_{0}-\bs^*\|_2^2+\|\bgamma_{0}-\bgamma^*\|_2^2\leq r^2$ by some initialization process. We can prove that $\|\bs^{(t)}-\bs^*\|_2^2+\|\bgamma^{(t)}-\bgamma^*\|_2^2\leq r^2$ for all $t\geq 0$ by induction. Specifically, assume it holds for $t$, then it suffices to ensure
\begin{align}
    &(3\eta+1/\mu)(\epsilon_1(k,n)^2+\epsilon_2(k,n)^2)\leq r^2(\mu/2-3(L^2+M^2)\eta),
\end{align}
which holds when $k$ is sufficiently large. Choosing $\eta$ to be sufficiently small, we can ensure that $1+3(L^2+M^2)\eta^2-\mu\eta/2\leq1$. In particular, we can set $\eta=\mu/(12(L^2+M^2))$, which implies
\begin{align*}
    \|\bs^{(t+1)}-\bs^*\|_2^2+\|\bgamma^{(t+1)}-\bgamma^*\|_2^2&\leq\rho\big(\|\bs^{(t)}-\bs^*\|_2^2+\|\bgamma^{(t)}-\bgamma^*\|_2^2\big)\\
    &\qquad+(3\eta^2+\eta/\mu)(\epsilon_1(k,n)^2+\epsilon_2(k,n)^2),
\end{align*}
with $\rho=1-\mu^2/(48(L^2+M^2))$. Therefore, we have
\begin{align*}
    \|\bs^{(t)}-\bs^*\|_2^2+\|\bgamma^{(t)}-\bgamma^*\|_2^2
    &\leq\rho^t\big(\|\bs_{0}-\bs^*\|_2^2+\|\bgamma_{0}-\bgamma^*\|_2^2\big)+\frac{3\eta^2+\eta/\mu}{1-\rho}(\epsilon_1(k,n)^2+\epsilon_2(k,n)^2)\\
    &\leq r^2\rho^t+\frac{\epsilon_1(k,n)^2+\epsilon_2(k,n)^2}{\mu^2},
\end{align*}
which completes the proof.
\end{proof}

\section{Proofs of Specific Examples}
In this section, we will provide the convergence analysis of Algorithm \ref{alg:PGD} for two specific examples with different noise distributions. In particular, we will show that Conditions \ref{assump:strong_convex} and \ref{assump:smooth} can be verified under these specific distributions. Recall the log-likelihood function
\begin{align}
    \cL\left(\bs,\bgamma;\bY\right)&= -\frac{1}{mk}\sum_{u=1}^{m} {\sum_{\left(i,j\right)\in\mathcal{D}_{u} }}\log  F\left(\gamma_{u}(s_i-s_j);Y_{ij}^{u}\right).
\end{align}
For the ease of presentation, we will omit $\Yb$ in the rest of the proof and assume that the observation set $\cD_u$ is parametrized by $k=|\cD_u|$ and vectors $\ab_{l,u}\in\RR^n$ for $l=1,\ldots,k$, where each $\ab_{l,u}=\eb_{i_l}-\eb_{j_l}$ for some pair of items $(i_l,j_l)$ that is compared by user $u$ and $\eb_i$ is the natural basis. Then, we can rewrite the loss function in terms of vector $\bs$ as follows
\begin{align}
    \cL\left(\bs,\bgamma\right)&= -\frac{1}{mk}\sum_{u=1}^{m} \sum_{l=1}^{k}\log  F\left(\gamma_{u}\ab_{l,u}^{\top}\bs;Y_{i_l j_l}^{u}\right).
\end{align}
Denote $g(x)=-\log F(x)$ for $x\in\RR$. Then we can calculate the gradient of loss function $\cL$ with respect to $\bs$ and $\bgamma$.
\begin{align}\label{eq:gradient_general}
\begin{split}
    \gradsb\cL(\bs,\bgamma)&=\frac{1}{mk}\sum_{u=1}^{m} \sum_{l=1}^{k}g'\left(\gamma_{u}\ab_{l,u}^{\top}\bs\right)\gamma_{u}\ab_{l,u},\\
    \gradgam\cL(\bs,\bgamma)&=\frac{1}{mk}\begin{bmatrix}
    \sum_{l=1}^{k}g'\left(\gamma_1\ab_{l,1}^{\top}\bs\right)\ab_{l,1}^{\top}\bs\\
    \vdots\\
    \sum_{l=1}^{k}g'\left(\gamma_{u}\ab_{l,u}^{\top}\bs\right)\ab_{l,u}^{\top}\bs\\
    \vdots
    \end{bmatrix}.
\end{split}
\end{align}
And the Hessian matrix can be calculated as
\begin{align}\label{eq:hessian_general}
\begin{split}
    \hessb\cL(\bs,\bgamma)&=\frac{1}{mk}\sum_{u=1}^{m} \sum_{l=1}^{k}g''\left(\gamma_{u}\ab_{l,u}^{\top}\bs\right)(\gamma_{u})^2\ab_{l,u}\ab_{l,u}^{\top},\\
    \hesgam\cL(\bs,\bgamma)&=\frac{1}{mk}\diag\begin{bmatrix}
    \sum_{l=1}^{k}g''\left(\gamma_1\ab_{l,1}^{\top}\bs\right)\ab_{l,1}^{\top}\bs\ab_{l,1}^{\top}\bs\\
    \vdots\\
    \sum_{l=1}^{k}g''\left(\gamma_{u}\ab_{l,u}^{\top}\bs\right)\ab_{l,u}^{\top}\bs\ab_{l,u}^{\top}\bs\\
    \vdots
    \end{bmatrix},
\end{split}
\end{align}
where $\diag(\xb)$ is the diagonal matrix with diagonal entries given by $\xb$.

\subsection{Proof of Heterogeneous BTL model}\label{sec:proof_converge_gumble}
Recall the definition in \eqref{eq:total-loss}. The loss function can be written as
\begin{align}
    \cL\left(\bs,\bgamma\right)&= \frac{1}{mk}\sum_{u=1}^{m} \sum_{l=1}^{k}g\left(\gamma_{u}\ab_{l,u}^{\top}\bs;Y_{i_l j_l}^{u}\right),
\end{align}
where $g(\cdot)$ is defined as
\begin{align}
    g(x;Y_{i_l j_l}^{u})=-\log  \frac{\exp(Y_{i_l j_l}^{u} x)}{1+\exp(x)}.
\end{align}
Therefore, the loss function of the HBTL model can be rewritten as follows:
\begin{align}
    \cL\left(\bs,\bgamma\right)&= \frac{1}{mk}\sum_{u=1}^{m} \sum_{l=1}^{k}\log\left(1+\exp(\gamma_{u}\ab_{l,u}^{\top}\bs)\right)-Y_{i_lj_l}^{u}\gamma_{u}\ab_{l,u}^{\top}\bs.
\end{align}
Recall the gradients and Hessian matrices calculated in \eqref{eq:gradient_general} and \eqref{eq:hessian_general}. We need to calculate $g'(\cdot)$ and $g''(\cdot)$. In particular, we have
\begin{align}\label{eq:g_deri_gumbel}
    g'(x;Y)&=\frac{-Y+(1-Y)\exp(x)}{1+\exp(x)},\qquad g''(x;Y)=\frac{\exp(x)}{(1+\exp(x))^2}.
\end{align}
It is easy to verify that $g'(x)$ is monotonically increasing on $\RR$. For any $|x|\leq\theta$, we have
\begin{align}\label{eq:bound_g_grad_gumble}
    \frac{-1}{1+e^{-\theta}}\leq g'(x;Y=1)\leq\frac{-1}{1+e^{\theta}},\qquad \frac{e^{-\theta}}{1+e^{-\theta}}\leq g'(x;Y=0)\leq\frac{e^{\theta}}{1+e^{\theta}}.
\end{align}
Furthermore, $g''(x)=g''(-x)$, $g''(x)$ is increasing on $(-\infty,0]$ and decreasing on $[0,\infty)$. Hence, for all $|x|\leq \theta$, we have
\begin{align}\label{eq:bound_g_hess_gumble}
e^{\theta}/(1+e^{\theta})^2\leq g''(x)\leq g''(0)=1/4.
\end{align}
We can further show that the following lemmas hold, which validates Conditions \ref{assump:strong_convex}, \ref{assump:smooth}, \ref{lemma:FOS} and \ref{lemma:staterr_general} used in the convergence analysis.

The first two lemmas verify the strong convexity and smoothness of $\cL$ with respect to $\bs$ and $\bgamma$ respectively.
\begin{lemma}\label{lemma:convex_smooth_sb_gumbel}
Suppose the noise $\epsilon$ follows the Gumbel distribution and the sample size $mk\geq 64(\gammamax+r)^2/(\gammamin-r)^2n\log n$. Let $r\leq\min\{\smax,\sqrt{\gammamax\smax}\}$, for all $\bs,\bs'\in\RR^n, \bgamma\in\RR^m$ such that $\|\bs-\bs^*\|_2\leq r, \|\bs'-\bs^*\|_2\leq r$ and $\|\bgamma-\bgamma^*\|_2\leq r$, we have
\begin{align*}
\cL(\bs,\bgamma)&\geq\cL(\bs',\bgamma)+\la\nabla_{\bs}\cL(\bs',\bgamma),\bs-\bs'\ra+\frac{\mu_1}{2}\|\bs-\bs'\|_2^2,\\
\cL(\bs,\bgamma)&\leq\cL(\bs',\bgamma)+\la\nabla_{\bs}\cL(\bs',\bgamma),\bs-\bs'\ra+\frac{L_1}{2}\|\bs-\bs'\|_2^2,
\end{align*}
where the coefficients are defined as
\begin{align*}
    \mu_1=\frac{(\gammamin-r)^2e^{5\gammamax\smax}}{n(1+e^{5\gammamax\smax})^2}, \qquad L_1=\frac{(\gammamax+r)^2}{n}.
\end{align*}
\end{lemma}

\begin{lemma}\label{lemma:convex_smooth_gam_gumbel}
Suppose the noise $\epsilon$ follows the Gumbel distribution and the sample size satisfies $k\geq 18(\smax+r)^4n^2/(m^2(\|\bs^*\|_2+r)^4)\log(mn)$. Let $r\leq\min\{\smax,\sqrt{\gammamax\smax}\}$, for all $\bs\in\RR^n, \bgamma,\bgamma'\in\RR^m$ such that $\|\bs-\bs^*\|_2\leq r$, $\bs^{\top}\one=0$, and $\|\bgamma-\bgamma^*\|_2\leq r,\|\bgamma'-\bgamma^*\|_2\leq r$, we have with probability at least $1-1/n$ that
\begin{align*}
\cL(\bs,\bgamma)&\geq\cL(\bs,\bgamma')+\la\nabla_{\bgamma}\cL(\bs,\bgamma'),\bgamma-\bgamma'\ra+\frac{\mu_2}{2}\|\bgamma-\bgamma'\|_2^2,\\
\cL(\bs,\bgamma)&\leq\cL(\bs,\bgamma')+\la\nabla_{\bgamma}\cL(\bs,\bgamma'),\bgamma-\bgamma'\ra+\frac{L_2}{2}\|\bgamma-\bgamma'\|_2^2,
\end{align*}
where the coefficients are defined as
\begin{align*}
    \mu_2=\frac{(\|\bs^*\|_2+r)^2e^{5\gammamax\smax}}{n(1+e^{5\gammamax\smax
    })^2},\quad
    L_2=\frac{(\|\bs^*\|_2+r)^2}{n}.
\end{align*}
\end{lemma}

\begin{lemma}\label{lemma:FOS_gumbel}
Let $r\leq\min\{\smax,\sqrt{\gammamax\smax}\}$, for all $\bs\in\RR^n, \bgamma\in\RR^m$ such that $\|\bs-\bs^*\|_2\leq r, \|\bs'-\bs^*\|_2\leq r$ and $\|\bgamma-\bgamma^*\|_2\leq r, \|\bgamma'-\bgamma^*\|_2\leq r$, we have
\begin{align*}
    \|\nabla_{\bs} \cL(\bs,\bgamma)-\nabla_{\bs} \cL(\bs,\bgamma')\|_2&\leq \frac{\sqrt{2}(1+2\gammamax\smax)}{\sqrt{m}}\|\bgamma-\bgamma^{\prime}\|_2,\\
    \|\gradgam \cL(\bs,\bgamma)-\gradgam\cL(\bs',\bgamma)\|_2&\leq \frac{\sqrt{2}(1+2\gammamax\smax)}{\sqrt{m}} \|\bs-\bs'\|_2.
\end{align*}
\end{lemma}

\begin{lemma}\label{lemma:staterr_gumbel}
Let $r\leq\min\{\smax,\sqrt{\gammamax\smax}\}$, for all $\bs\in\RR^n, \bgamma\in\RR^m$ such that $\|\bs-\bs^*\|_2\leq r$ and $\|\bgamma-\bgamma^*\|_2\leq r$. Denote $\bar\cL$ as the expected loss which takes expectation of $\cL$ over the random choice of comparison pair. We have
\begin{align*}
    \|\gradsb\cL(\bs,\bgamma)-\gradsb\bar\cL(\bs,\bgamma)\|_2&\leq\epsilon_{1}(k,n):=\frac{2(\gammamax+r)}{1+e^{-5\gammamax\smax}}\sqrt{\frac{2\log(2n)}{mk}},\\
    \|\gradgam\cL(\bs,\bgamma)-\gradgam\bar\cL(\bs,\bgamma)\|_2&\leq\epsilon_{2}(k,n):=\frac{10\gammamax\smax}{1+e^{5\gammamax\smax}}\sqrt{\frac{2\log(2mn)}{mk}},
\end{align*}
holds with probability at least $1-1/n$.
\end{lemma}

\begin{proof}[Proof of Corollary \ref{coro:convergence_gumbel}]
Now we prove the convergence of Algorithm \ref{alg:PGD} for Gumbel noise. Our proof will be similar to that of Theorem \ref{thm:convergence_general}. In particular, we only need to verify that Conditions \ref{assump:strong_convex}, \ref{assump:smooth}, \ref{lemma:FOS} and \ref{lemma:staterr_general} hold when the noise follows a Gumbel distribution. According to Lemmas \ref{lemma:convex_smooth_sb_gumbel} and \ref{lemma:convex_smooth_gam_gumbel}, we know that $\cL(\bs,\bgamma)$ is $\mu_1$-strongly convex and $L_1$-smooth with respect to $\bs$, and is $\mu_2$-strongly convex and $L_2$-smooth with respect to $\bgamma$. More specifically, when $mk\geq 64n\log(n)$, we have
\begin{align}
    \mu_1\geq (\gammamin-r)^2e^{5C_0}/(n(1+e^{5C_0})^2),\qquad L_1\leq(\gammamax+r)^2/n,
\end{align}
where we use the fact that $\gammamax\smax=C_0$. In addition, note that $\smax\leq\sqrt{m}/n\|\bs^*\|_s$ and $\|\bs^{(t)}-\bs^*\|\leq r$. Hence if $mk\geq 18\log(mn)$, we have
\begin{align}
    \mu_2\geq (\|\bs^*\|_2+r)^2e^{5C_0}/(n(1+e^{5C_0})^2),\qquad L_2\leq(\|\bs^*\|_2+r)^2/n
\end{align}
By Lemma \ref{lemma:FOS_gumbel} and the assumption that $C_0\geq 1/2$, we know that $\cL$ satisfies the first-order stability (Condition \ref{lemma:FOS}) with $M_1=M_2=4\sqrt{2}\gammamax\smax/\sqrt{m}$. Note that by assumption, we have
\begin{align*}
    \smax\leq
    \frac{\gammamin e^{2C_0}}{16\sqrt{2}\gammamax(1+e^{2C_0})^2}\frac{\sqrt{m}\|\bs^*\|_2}{n}.
\end{align*}
This immediately implies that $M=M_1=M_2\leq\sqrt{\mu_1\mu_2}/4$. Therefore, by similar arguments as in the proof of Theorem \ref{thm:convergence_general}, we need to set step sizes $\eta_1=\eta_2=\eta<\mu/(6(L^2+M^2))$, where $\mu=\min\{\mu_1,\mu_2\}$, $L=\max\{L_1,L_2\}$. In fact, it suffices to set
\begin{align*}
    \eta<\frac{mne^{5C_0}\Gamma_1^2}{6(1+e^{5C_0})^2(m\Gamma_2^4+32n^2C_0^2)},
\end{align*}
with $\Gamma_1=\min\{\gammamin/2,\|\bs^*\|_2\}$ and $\Gamma_2=\max\{2\gammamax,2\|\bs^*\|_2\}$. We thus obtain
\begin{align*}
    \|\bs^{(t)}-\bs^*\|_2^2+\|\bgamma^{(t)}-\bgamma^*\|_2^2
    &\leq r^2\rho^t+\frac{\epsilon_1(k,n)^2+\epsilon_2(k,n)^2}{\mu^2}\leq r^2\rho^t+\frac{\Lambda n^2\log(4mn^2)}{mk},
\end{align*}
where $\rho=1-\eta(\mu-6\eta(\Gamma_2^4/n^2+32C_0^2/m))/2$ and the last inequality comes from Lemma \ref{lemma:staterr_gumbel} with the constant $\Lambda$ defined as follows:
\begin{align*}
    \Lambda=\max\bigg\{\frac{200C_0^2(1+e^{5C_0})^2}{\Gamma_1^4e^{10C_0}},\frac{8(\gammamax+r)^2(1+e^{5C_0})^4}{\Gamma_1^4(1+e^{-5C_0})^2}\bigg\}.
\end{align*}
This completes the proof.
\end{proof}

\subsection{Proof of Heterogeneous Thurstone Case V model}
In this subsection, we provide the analysis of our algorithm when the noise $\epsilon_i$ follows a Gaussian distribution, which results in the Thurstone model. The log-likelihood function can be written as
\begin{align}
    \cL\big(\bs,\bgamma;\bY\big)&= \frac{1}{mk}\sum_{u=1}^{m} \sum_{l=1}^{k}g(\gamma_{u}\ab_{l,u}^{\top}\bs;Y_{i_l j_l}^{u}).
\end{align}
with $g(\cdot)$ defined as $g(x)=-\log  \Phi(x)$ with $\Phi(\cdot)$ be the CDF of the standard normal distribution. Note that $\Pr(Y_{i_l j_l}^{u}=1)=\Phi(\gamma_{u}\ab_{l,u}^{\top}\bs)$ and $\Pr(Y_{i_l j_l}^{u}=0)=1-\Phi(\gamma_{u}\ab_{l,u}^{\top}\bs)=\Phi(-\gamma_{u}\ab_{l,u}^{\top}\bs)$. Thus we can write $g(\cdot)$ as $g(\gamma_{u}\ab_{l,u}^{\top}\bs;Y_{i_lj_l}^{u})=-\log\Phi((2Y_{i_lj_l}^{u}-1)\gamma_{u}\ab_{l,u}^{\top}\bs)$. Note that $(2Y-1)^2=1$, we have
\begin{align*}
    g'(x;Y)=-\frac{(2Y-1)\Phi'(x)}{\Phi(x)}, \qquad g''(x;Y)=\frac{\Phi'(x)^2-\Phi(x)\Phi''(x)}{\Phi(x)^2}.
\end{align*}
In order to bound $g'(x)$ and $g''(x)$, we first calculate the derivatives of $\Phi(x)$ as follows:
\begin{align}\label{eq:normal_CDF}
    \Phi(x)=\int_{-\infty}^{x}\frac{1}{\sqrt{2\pi}}e^{-\frac{z^2}{2}}\dd z,\qquad
    \Phi'(x) = \frac{1}{\sqrt{2\pi}}e^{-\frac{x^2}{2}},\qquad
    \Phi''(x)=\frac{-x}{\sqrt{2\pi}}e^{-\frac{x^2}{2}}.
\end{align}
For any $\theta>0$ such that $|x|\leq\theta$, we have
\begin{align*}
    \frac{e^{-\theta^2/2}}{\sqrt{2\pi}\Phi(\theta)}\leq|g'(x)|\leq\frac{1}{\sqrt{2\pi}\Phi(-\theta)}.
\end{align*}
We can verify that $g''(x)$ is monotonically decreasing on $\RR^d$ and $g''(x)>0$ also always hold. Thus for all $|x|\leq\theta$, we have $g''(\theta)\leq g''(x)\leq g''(-\theta)$.

\begin{proof}[Proof of Corollary \ref{coro:convergence_normal}]
Recall the derivation of the gradient in \eqref{eq:gradient_general} and the Hessian in \eqref{eq:hessian_general} of the loss function $\cL$. In order to verify Conditions \ref{assump:strong_convex}, \ref{assump:smooth}, \ref{lemma:FOS} and \ref{lemma:staterr_general}, we only need the upper and lower bounds of $g'(\gamma_{u}\ab_{l,u}^{\top}\bs;Y_{i_lj_l}^{u})$ for all $u=1,\ldots,m$ and $l=1,\ldots,k$. Therefore, using exactly the same proof techniques as in Section \ref{sec:proof_converge_gumble}, we can also establish strong convexity, smoothness, first-order stability and the statistical error bound for sample loss function $\cL$ when the noise $\epsilon$ follows the standard normal distribution. We omit the proof since it is the same as that of the Gumbel case. We can verify that $\cL$ is $\mu_1$-strongly convex and $L_1$-smooth with respect to $\bs$, and is $\mu_2$-strongly convex and $L_2$-smooth with respect to $\bgamma$. The coefficient parameters are defined as $ \mu_1=(\gammamin-r)^2H(5C_0)/n, L_1=(\gammamax+r)^2H(-5C_0)/n$, $\mu_2=(\|\bs^*\|_2+r)^2H(5C_0)/n$ and $L_2=(\|\bs^*\|_2+r)^2H(-5C_0)/n$. Note that $H(x)$ is a function defined based on the normal CDF $\Phi(\cdot)$:
\begin{align*}
   H(x)=[\Phi'(x)^2-\Phi(x)\Phi''(x)]/\Phi(x)^2,
\end{align*}
where $\Phi,\Phi',\Phi''$ are defined in \eqref{eq:normal_CDF}. The loss function $\cL$ also satisfies Condition \ref{lemma:FOS} with $M=M_1=M_2=(1/\Phi(-5C_0)+5\sqrt{2\pi}H(-5C_0)\gammamax\smax)/\sqrt{m\pi}$. In order to make sure that $M\leq\sqrt{\mu_1\mu_2}/4$, we only need $\smax\leq\sqrt{\pi}\gammamin H(5C_0)/[4\gammamax(2/\Phi(-5C_0))+5\sqrt{2\pi}H(-5C_0)]\cdot\sqrt{m}\|\bs^*\|_2/n$. Therefore, by Theorem \ref{thm:convergence_general}, if we choose step sizes $\eta_1=\eta_2=\eta$ such that
\begin{align*}
    \eta&<\frac{mn \Gamma_1^2H(5C_0)}{6(m\Gamma_2^4+50n^2C_0^2)H(-5C_0)^2},\\
    &\text{with } \Gamma_2=\min\{\gammamin/2,\|\bs^*\|_2\}, \quad\Gamma_2=\max\{2\gammamax,2\|\bs^*\|_2\},
\end{align*}
then we are able to obtain the following convergence result:
\begin{align}\label{eq:contraction_normal}
    \|\bs^{(t)}-\bs^*\|_2^2+\|\bgamma^{(t)}-\bgamma^*\|_2^2
    &\leq r^2\rho^t+\frac{\epsilon_1(k,n)^2+\epsilon_2(k,n)^2}{\mu^2},
\end{align}
where $\mu=\Gamma_1^2H(5C_0)/n$, $\rho=1-\eta(\mu-6\eta(\Gamma_2^4/n^2+32C_0^2/m))/2$ and $\epsilon_1(k,n),\epsilon_2(k,n)$ are the statistical error bounds. Similar to the proof of Lemma \ref{lemma:staterr_gumbel}, we know that $\epsilon_1(k,n)=(\gammamax+r)/(\sqrt{\pi}\Phi(-5C_0))\sqrt{2\log(2n)/(mk)}$ and $\epsilon_2(k,n)=10\gammamax\smax/(\sqrt{\pi}\Phi(-5C_0))\sqrt{\log(2mn)/(mk)}$. Plugging these two bounds into \eqref{eq:contraction_normal} yields
\begin{align*}
    \|\bs^{(t)}-\bs^*\|_2^2+\|\bgamma^{(t)}-\bgamma^*\|_2^2
    &\leq r^2\rho^t+\frac{\Lambda' n^2\log(4mn^2)}{mk},
\end{align*}
which holds with probability at least $1-1/n$, where $\Lambda'$ is a constant defines as follows.
\begin{align*}
    \Lambda'=\frac{2\max\{(\gammamax+r)^2,50 C_0^2\}}{\pi\Gamma_1^4 H(5C_0)^2\Phi(-5C_0)^2}.
\end{align*}
This completes the proof.
\end{proof}

\section{Proofs of Technical Lemmas}
In this section, we provide the proofs of technical lemmas used in the previous section.

\subsection{Proof of Lemma \ref{lemma:convex_smooth_sb_gumbel}}
We first lay down the following useful lemma.
\begin{lemma}\citep{tropp2012user}\label{lemma:matrix_bern}
Consider a sequence of i.i.d. random matrices $\{\Xb_k\}$ in $\RR^{d\times d}$ with $\EE[\Xb_k]=\zero$ and $\|\Xb_k\|_2\leq R$. Then for all $t\geq 0$
\begin{align*}
    \Pr\bigg(\bigg\|\sum_{k}\Xb_k\bigg\|\geq t\bigg)\leq d\exp\bigg(-\frac{t^2}{2\sigma^2+2Rt/3}\bigg),
\end{align*}
where $\sigma^2=\|\sum_{k}\EE[\Xb_k^2]\|_2$.
\end{lemma}

\begin{proof}[Proof of Lemma \ref{lemma:convex_smooth_sb_gumbel}]
Using Taylor expansion, we have
\begin{align*}
\cL(\bs,\bgamma)=\cL(\bs',\bgamma)+\la\nabla_{\bs}\cL(\bs',\bgamma),\bs-\bs'\ra+\frac{1}{2}(\bs-\bs')^{\top}\hessb\cL(\tilde\bs,\bgamma)(\bs-\bs'),
\end{align*}
where $\tilde\bs=\bs+\theta(\bs'-\bs)$ for some $\theta\in(0,1)$.
In order to show the strong convexity and smoothness of $\cL$, we need to bound the minimal and maximum eigenvalues of $\hessb\cL(\bs,\bgamma)$. Note that $\bs,\bgamma$ lie in a neighborhood with radius $r$ of the true parameters $\bs^*,\bgamma^*$ respectively. When $r\leq\min\{\smax,\sqrt{\gammamax\smax}\}$, we have
\begin{align}\label{eq:bound_gamma_a_s_sample}
    |\gamma_{u}\ab_{l,u}^{\top}\bs|\leq|(\gamma_{u}-\gamma_u^{*})\ab_{l,u}^{\top}(\bs-\bs^*)|+|\gamma_u^{*}\ab_{l,u}^{\top}(\bs-\bs^*)|+|\gamma_u^{*}\ab_{l,u}^{\top}\bs^*|\leq 5\gammamax\smax.
\end{align}
For any $\bDelta\in\RR^n$, we have
\begin{align*}
    \frac{1}{mk}\sum_{u=1}^{m} \sum_{l=1}^{k}\frac{(\gamma_{u})^2\exp(5\gammamax\smax)}{(1+\exp(5\gammamax\smax))^2}\bDelta^{\top}\ab_{l,u}\ab_{l,u}^{\top}\bDelta
    &\leq\bDelta^{\top}\hessb\cL(\bs,\bgamma)\bDelta\\
    &=\frac{1}{mk}\sum_{u=1}^{m} \sum_{l=1}^{k}g''\left(\gamma_{u}\ab_{l,u}^{\top}\bs\right)(\gamma_{u})^2\bDelta^{\top}\ab_{l,u}\ab_{l,u}^{\top}\bDelta\\
    &\leq\frac{1}{4mk}\sum_{u=1}^{m} \sum_{l=1}^{k}(\gamma_{u})^2\bDelta^{\top}\ab_{l,u}\ab_{l,u}^{\top}\bDelta,
\end{align*}
where we used the monotonicity of $g''$. Since $\ab_{l,u}=\eb_{i_l}-\eb_{j_l}$ and $i_l,j_l$ are uniformly distributed, we have $\EE[\ab_{l,u}\ab_{l,u}^{\top}]=\EE[\eb_{i_l}\eb_{i_l}^{\top}+\eb_{j_l}\eb_{j_l}^{\top}-\eb_{i_l}\eb_{j_l}^{\top}-\eb_{j_l}\eb_{i_l}^{\top}]=2/n\Ib-2/n(\one\one^{\top}/n)$. We define
\begin{align}\label{eq:def_L_matrix}
    \Xb_{l,u}=(\gamma^{u})^2\bigg[\ab_{l,u}\ab_{l,u}^{\top}-\frac{2(\Ib-\one\one^{\top}/n)}{n}\bigg], \quad \Lb=\frac{2(\Ib-\one\one^{\top}/n)}{n}.
\end{align}
Thus we have $\EE[\Xb_{l,u}]=\zero$. Furthermore, we have $\|\Xb_{l,u}\|_2\leq 2(\gammamax+r)^2$ and $\EE[\Xb_{l,u}^2] \leq4(\gammamax+r)^4(n-1)/n^2(\Ib-\one\one^{\top}/n)$. Applying Lemma \ref{lemma:matrix_bern} yields
\begin{align*}
    \Pr\bigg(\bigg\|\frac{1}{mk}\sum_{u=1}^{m} \sum_{l=1}^{k}\Xb_{l,u}\bigg\|_2\geq t\bigg)
    &\leq 2n\exp\bigg(\frac{-t^2}{8(\gammamax+r)^4(n-1)/(n^2mk)+4t(\gamma_{\max}+r)^2/(3mk)}\bigg)\\
    &\leq 2n\exp\bigg(\frac{-t^2}{8(\gammamax+r)^4/(nmk)+4t(\gamma_{\max}+r)^2/(3mk)}\bigg),
\end{align*}
which implies that
\begin{align*}
    \bigg\|\frac{1}{mk}\sum_{u=1}^{m} \sum_{l=1}^{k}\Xb_{l,u}\bigg\|_2&\leq\frac{8(\gammamax+r)^2\log n}{3mk}+4(\gammamax+r)^2\sqrt{\frac{\log n}{nmk}}\\
    &\leq 8(\gammamax+r)^2\sqrt{\frac{\log n}{nmk}}
\end{align*}
holds with probability at least $1-1/n$, where the last inequality holds when $mk\geq 4/9n\log n$. Therefore, we have
\begin{align*}
    \|\hessb\cL(\bs,\bgamma)\|_2\leq (\gammamax+r)^2\bigg(\frac{1}{2n}+2\sqrt{\frac{\log n}{nmk}}\bigg)\leq \frac{(\gammamax+r)^2}{n}.
\end{align*}
On the other hand, for any $\bDelta\in\RR^n$ such that $\bDelta^{\top}\one=0$, we have
\begin{align*}
    \frac{1}{mk}\sum_{u=1}^{m} \sum_{l=1}^{k}\bDelta^{\top}\Xb_{l,u}\bDelta\geq -8\gammamax^2\sqrt{\frac{\log n}{nmk}}\|\bDelta\|_2^2,
\end{align*}
which implies
\begin{align*}
    \bDelta^{\top}\hessb\cL(\bs,\bgamma)\bDelta\geq \bigg(\frac{2(\gammamin-r)^2}{n}-8(\gammamax+r)^2\sqrt{\frac{\log }{nmk}}\bigg)\|\bDelta\|_2^2.
\end{align*}
Therefore, when $k$ is sufficiently large such that $mk\geq 64(\gammamax+r)^2/(\gammamin-r)^2n\log n$, we have
\begin{align*}
    \lambda_{\min}\big(\hessb\cL(\bs,\bgamma)\big)\geq\frac{(\gammamax-r)^2e^{5\gammamax\smax}}{n(1+e^{5\gammamax\smax})^2}.
\end{align*}
This completes the proof.
\end{proof}

\subsection{Proof of Lemma \ref{lemma:convex_smooth_gam_gumbel}}
\begin{proof}
Using Taylor expansion, we get
\begin{align}
    \cL(\bs,\bgamma)=\cL(\bs,\bgamma')+\la\nabla_{\bgamma}\cL(\bs,\bgamma'),\bgamma-\bgamma'\ra+\frac{1}{2}(\bgamma-\bgamma')^{\top}\hesgam\cL(\bs,\tilde\bgamma)(\bgamma-\bgamma'),
\end{align}
where $\tilde\bgamma=\bgamma+\theta(\bgamma'-\bgamma)$ for some $\theta\in(0,1)$. Recall the Hessian matrix with respect to $\bgamma$:
\begin{align*}
    \hesgam\cL(\bs,\bgamma)&=\frac{1}{mk}\diag\begin{bmatrix}
    \sum_{l=1}^{k}g''\left(\gamma_1\ab_{l,1}^{\top}\bs\right)\ab_{l,1}^{\top}\bs\ab_{l,1}^{\top}\bs\\
    \vdots\\
    \sum_{l=1}^{k}g''\left(\gamma_{u}\ab_{l,u}^{\top}\bs\right)\ab_{l,u}^{\top}\bs\ab_{l,u}^{\top}\bs\\
    \vdots
    \end{bmatrix}.
\end{align*}
For any fixed $u$, we denote $X_{l,u}=\ab_{l,u}^{\top}\bs\ab_{l,u}^{\top}\bs-\bs^{\top}\Lb\bs$, where $\Lb$ is defined as in \eqref{eq:def_L_matrix}. Recall the calculation of $g''$ in \eqref{eq:g_deri_gumbel},\eqref{eq:bound_g_hess_gumble} and that $|\gamma_{u}\ab_{l,u}^{\top}\bs|\leq 5\gammamax\smax$ by \eqref{eq:bound_gamma_a_s_sample}, we have
\begin{align*}
    \frac{e^{5\gammamax\smax}}{(1+e^{5\gammamax\smax
    })^2}\leq g''\left(\gamma_{u}\ab_{l,u}^{\top}\bs\right)=\frac{\exp(\gamma_{u}\ab_{l,u}^{\top}\bs)}{\big(1+\exp(\gamma_{u}\ab_{l,u}^{\top}\bs)\big)^2}\leq\frac{1}{4}.
\end{align*}
Since $\hesgam\cL(\bs,\bgamma)$ is a diagonal matrix, the eigenvalues of $\hesgam\cL(\bs,\bgamma)$ can be bounded by
\begin{align}
    \frac{e^{5\gammamax\smax}}{(1+e^{5\gammamax\smax
    })^2}\min_{u}\frac{1}{mk}\sum_{l=1}^{k}\big(\ab_{l,u}^{\top}\bs\big)^2&\leq\lambda_{\min}(\hesgam\cL(\sb,\bgamma))\notag\\
    &\leq\lambda_{\max}(\hesgam\cL(\sb,\bgamma))\notag\\
    &\leq\frac{1}{4}\max_{u}\frac{1}{mk}\sum_{l=1}^{k}\big(\ab_{l,u}^{\top}\bs\big)^2.
\end{align}
Since $\bs^{\top}\one=0$, it is easy to verify $\EE[X_{l,u}]=\EE[\bs^{\top}(\ab_{l,u}\ab_{l,u}^{\top}-\Lb)\bs]=0$ and $|X_{l,u}|\leq 6(\smax+r)^2$. For any fixed $u$, applying Hoeffding's inequality yields
\begin{align*}
    \Pr\bigg(-\frac{1}{mk}\sum_{l=1}^{k}X_{l,u}\geq t\bigg)=\Pr\bigg(\frac{1}{mk}\sum_{l=1}^{k}X_{l,u}\geq t\bigg)\leq \exp\bigg(-\frac{m^2t^2k}{18(\smax+r)^4}\bigg).
\end{align*}
Further applying union bound, we have
\begin{align*}
    \Pr\bigg(\max_{u}\frac{1}{mk}\sum_{l=1}^{k}X_{l,u}\geq t\bigg)\leq\sum_{u}\Pr\bigg(\frac{1}{k}\sum_{l=1}^{k}X_{l,u}\geq mt\bigg)\leq m\exp\bigg(-\frac{m^2t^2k}{18(\smax+r)^4}\bigg),
\end{align*}
which immediately implies that
\begin{align}
  \lambda_{\max}(\hesgam\cL(\bs,\bgamma))&\leq\frac{1}{4}\max_{u}\frac{1}{mk}\sum_{l=1}^{k}\ab_{l,u}^{\top}\bs\ab_{l,u}^{\top}\bs \notag\\
  &\leq \frac{(\|\bs^*\|_2+r)^2}{2n}+\frac{3(\smax+r)^2}{4m}\sqrt{\frac{2\log(mn)}{k}}\notag\\
  &\leq\frac{(\|\bs^*\|_2+r)^2}{n}
\end{align}
holds with probability at least $1-1/n$, where the last inequality is true when the sample size satisfies
$k\geq 5(\smax+r)^4n^2/(m^2(\|\bs^*\|_2+r)^4)\log(mn)$. On the other hand, we also have
\begin{align*}
    \Pr\bigg(\max_{u}-\frac{1}{mk}\sum_{l=1}^{k}X_{l,u}\geq t\bigg)\leq\sum_{u}\Pr\bigg(-\frac{1}{k}\sum_{l=1}^{k}X_{l,u}\geq mt\bigg)\leq m\exp\bigg(-\frac{m^2t^2k}{18(\smax+r)^4}\bigg),
\end{align*}
which leads to the conclusion that
\begin{align}
    \lambda_{\min}(\hesgam\cL(\bs,\bgamma))&\geq\frac{e^{5\gammamax\smax}}{(1+e^{5\gammamax\smax
    })^2}\max_{u}\frac{1}{mk}\sum_{l=1}^{k}\ab_{l,u}^{\top}\bs\ab_{l,u}^{\top}\bs \notag\\
    &\geq \frac{e^{5\gammamax\smax}}{(1+e^{5\gammamax\smax
    })^2}\bigg(\frac{2(\|\bs^*\|_2+r)^2}{n}-\frac{3(\smax+r)^2}{m}\sqrt{\frac{2\log(mn)}{k}}\bigg)\notag\\
    &\geq\frac{(\|\bs^*\|_2+r)^2e^{5\gammamax\smax}}{n(1+e^{5\gammamax\smax
    })^2}
\end{align}
holds with probability at least $1-1/n$, where the last inequality is due to $k\geq 18(\smax+r)^4n^2/(m^2(\|\bs^*\|_2+r)^4)\log(mn)$.
\end{proof}

\subsection{Proof of Lemma \ref{lemma:FOS_gumbel}}
\begin{proof}
Recall the gradient of $\cL$ with respect to $\bs$ in \eqref{eq:gradient_general}. It holds that
\begin{align*}
    \|\nabla_{\bs} \cL(\bs,\bgamma)-\nabla_{\bs} \cL(\bs,\bgamma')\|_2&=\bigg\|\frac{1}{mk}\sum_{u=1}^{m} \sum_{l=1}^{k}\big(g'\left(\gamma_{u}\ab_{l,u}^{\top}\bs\right)\gamma_{u}-g'\left(\gamma_{u}^{\prime}\ab_{l,u}^{\top}\bs\right)\gamma_{u}^{\prime}\big)\ab_{l,u}\bigg\|_2\\
    &\leq\frac{1}{mk}\sum_{u=1}^{m} \sum_{l=1}^{k}\big[\big|g'\left(\gamma_{u}\ab_{l,u}^{\top}\bs\right)(\gamma_{u}-\gamma_{u}^{\prime})\big|\\
    &\qquad+\big|\big(g'\left(\gamma_{u}\ab_{l,u}^{\top}\bs\right)-g'\left(\gamma_{u}^{\prime}\ab_{l,u}^{\top}\bs\right)\big)\gamma_{u}^{\prime}\big|\big]\|\ab_{l,u}\|_2.
\end{align*}
Note that we have  $|g'(\gamma_{u}\ab_{l,u}^{\top}\bs)|\leq1$ and $\|\ab_{l,u}\|_2=\sqrt{2}$. In addition, by the mean value theorem we have
\begin{align*}
    g'\left(\gamma_{u}\ab_{l,u}^{\top}\bs\right)-g'\left(\gamma_{u}^{\prime}\ab_{l,u}^{\top}\bs\right)=g''(x)(\gamma_{u}-\gamma_{u}^{\prime})\ab_{l,u}^{\top}\bs,
\end{align*}
where $x=t\gamma_{u}\ab_{l,u}^{\top}\bs+(1-t)\gamma_{u}^{\prime}\ab_{l,u}^{\top}\bs$ for some $t\in(0,1)$. By plugging the range of $\gamma_{u}$ and $\bs$, we have $|x|\leq 5\gammamax\smax$ by \eqref{eq:bound_gamma_a_s_sample} and hence $|g''(x)|=|e^x/(1+e^x)^2|\leq 1/4$. Now we can bound $\|\nabla_{\bs} \cL(\bs,\bgamma)-\nabla_{\bs} \cL(\bs,\bgamma')\|_2$ as follows:
\begin{align*}
    \|\nabla_{\bs} \cL(\bs,\bgamma)-\nabla_{\bs} \cL(\bs,\bgamma')\|_2&\leq\frac{1}{mk}\sum_{u=1}^{m} \sum_{l=1}^{k}\sqrt{2}(1+2\gammamax\smax)|\gamma_{u}-\gamma_{u}^{\prime}|\\
    &\leq\frac{\sqrt{2}(1+2\gammamax\smax)}{\sqrt{m}}\|\bgamma-\bgamma^{\prime }\|_2.
\end{align*}
Now we prove the upper bound of $\|\gradgam\cL(\bs,\bgamma)-\gradgam\cL(\bs',\bgamma)\|_2$. First, we have by \eqref{eq:gradient_general} that
\begin{align*}
    \gradgam\cL(\bs,\bgamma)-\gradgam\cL(\bs',\bgamma)&=\frac{1}{mk}\begin{bmatrix}
    \sum_{l=1}^{k}\ab_{l,1}^{\top}\big(g'\big(\gamma_1\ab_{l,1}^{\top}\bs\big)\bs-g'\big(\gamma_1\ab_{l,1}^{\top}\bs'\big)\bs'\big)\\
    \vdots\\
    \sum_{l=1}^{k}\ab_{l,u}^{\top}\big(g'\big(\gamma_{u}\ab_{l,u}^{\top}\bs\big)\bs-g'\big(\gamma_{u}\ab_{l,u}^{\top}\bs'\big)\bs'\big)\\
    \vdots
    \end{bmatrix}.
\end{align*}
Note that for each $u$, we have
\begin{align}\label{eq:fos_decop}
    &\ab_{l,u}^{\top}\big(g'\big(\gamma_{u}\ab_{l,u}^{\top}\bs\big)\bs-g'\big(\gamma_{u}\ab_{l,u}^{\top}\bs'\big)\bs'\big)\notag\\
    &=\ab_{l,u}^{\top}\big[g'\big(\gamma_{u}\ab_{l,u}^{\top}\bs\big)(\bs-\bs')+\big(g'\big(\gamma_{u}\ab_{l,u}^{\top}\bs\big)-g'\big(\gamma_{u}\ab_{l,u}^{\top}\bs'\big)\big)\bs'\big].
\end{align}
For the first term in \eqref{eq:fos_decop}, we have
\begin{align*}
    \big|\ab_{l,u}^{\top}g'\big(\gamma_{u}\ab_{l,u}^{\top}\bs\big)(\bs-\bs')\big|\leq\sqrt{2}\|\bs-\bs'\|_2.
\end{align*}
For the second term in \eqref{eq:fos_decop}, applying the mean value theorem yields
\begin{align*}
    \big|\ab_{l,u}^{\top}\big(g'\big(\gamma_{u}\ab_{l,u}^{\top}\bs\big)-g'\big(\gamma_{u}\ab_{l,u}^{\top}\bs'\big)\big)\bs'\big|=\big|g''(x)\gamma_{u}\ab_{l,u}^{\top}\big(\bs-\bs'\big)\ab_{l,u}^{\top}\bs'\big|\leq\frac{5\sqrt{2}\gammamax\smax}{4}\|\bs-\bs'\|_2,
\end{align*}
where $x=t\gamma_{u}\ab_{l,u}^{\top}\bs+(1-t)\gamma^{u}\ab_{l,u}^{\top}\bs'$ for some $t\in(0,1)$. Therefore, we have
\begin{align*}
    \|\gradgam\cL(\bs,\bgamma)-\gradgam\cL(\bs',\bgamma)\|_2\leq\frac{\sqrt{2}(1+2\gammamax\smax)}{\sqrt{m}}\|\bs-\bs'\|_2,
\end{align*}
which completes our proof.
\end{proof}

\subsection{Proof of Lemma \ref{lemma:staterr_gumbel}}
\begin{proof}
According to \eqref{eq:gradient_general}, the gradient of $\cL$ with respect to $\bs$ is
\begin{align*}
    \gradsb\cL(\bs,\bgamma)
    &=\frac{1}{mk}\sum_{u}\sum_{l}\frac{\big(-Y+(1-Y)\exp(\gamma_{u}\ab_{l,u}^{\top}\bs)\big)\gamma_{u}\ab_{l,u}}{1+\exp(\gamma_{u}\ab_{l,u}^{\top}\bs)}.
\end{align*}
By assumption we have $|\gamma_{u}|\leq (\gammamax+r)$ and $|\gamma_{u}\ab_{l,u}^{\top}\bs|\leq 5\gammamax\smax$ by \eqref{eq:bound_gamma_a_s_sample}. In addition, we have $\|\gamma_{u}\ab_{l,u}/(1+\exp(\gamma_{u}\ab_{l,u}\bs))\|_2\leq \sqrt{2}(\gammamax+r)/(1+e^{-5\gammamax\smax})$. Applying Hoeffding's inequality, we have
\begin{align*}
    \Pr\big(\|\gradsb\cL(\bs,\bgamma)-\gradsb\bar\cL(\bs,\bgamma)\|_2\geq t\big)\leq 2\exp\bigg(\frac{-(1+e^{-5\gammamax\smax})^2mkt^2}{8(\gammamax+r)^2}\bigg),
\end{align*}
which implies that
\begin{align*}
    \|\gradsb\cL(\bs,\bgamma)-\gradsb\bar\cL(\bs,\bgamma)\|_2\leq\frac{2(\gammamax+r)}{1+e^{-5\gammamax\smax}}\sqrt{\frac{2\log(2n)}{mk}}
\end{align*}
holds with probability at least $1-1/n$. Recall the calculation in \eqref{eq:gradient_general}, the gradient of $\cL$ with respect to $\bgamma$ is
\begin{align*}
    \gradgam\cL(\bs,\bgamma)&=\frac{1}{mk}\begin{bmatrix}
    \sum_{l=1}^{k}g'\left(\gamma_1\ab_{l,1}^{\top}\bs\right)\ab_{l,1}^{\top}\bs\\
    \vdots\\
    \sum_{l=1}^{k}g'\left(\gamma_{u}\ab_{l,u}^{\top}\bs\right)\ab_{l,u}^{\top}\bs\\
    \vdots
    \end{bmatrix}.
\end{align*}
The squared statistical error is
\begin{align*}
    \|\gradgam\cL(\bs,\bgamma)-\gradgam\bar\cL(\bs,\bgamma)\|_2
    &=\frac{1}{mk}\sqrt{\sum_{u}\bigg[\sum_{l}\big(g'(\gamma_{u}\ab_{l,u}^{\top}\bs)\ab_{l,u}-\EE[g'(\gamma_{u}\ab_{l,u}^{\top}\bs)\ab_{l,u}]\big)^{\top}\bs\bigg]^2},
\end{align*}
which implies for all $t\geq 0$
\begin{align*}
    &\Pr\big(\|\gradgam\cL(\bs,\bgamma)-\gradgam\bar\cL(\bs,\bgamma)\|_2\geq t\big)\\
    &\leq\Pr\bigg(\max_{u}\frac{1}{k}\sum_{l}\big(g'(\gamma_{u}\ab_{l,u}^{\top}\bs)\ab_{l,u}-\EE[g'(\gamma_{u}\ab_{l,u}^{\top}\bs)\ab_{l,u}]\big)^{\top}\bs\geq\sqrt{m}t\bigg)\\
    &\leq\sum_{u}\Pr\bigg(\frac{1}{k}\sum_{l}\big(g'(\gamma_{u}\ab_{l,u}^{\top}\bs)\ab_{l,u}-\EE[g'(\gamma_{u}\ab_{l,u}^{\top}\bs)\ab_{l,u}]\big)^{\top}\bs\geq\sqrt{m}t\bigg),
\end{align*}
where the last inequality is due to union bound. For each user $u$, we have
\begin{align*}
   |(g'(\gamma_{u}\ab_{l,u}^{\top}\bs)\ab_{l,u}-\EE[g'(\gamma_{u}\ab_{l,u}^{\top}\bs)\ab_{l,u}])^{\top}\bs|\leq \frac{10\gammamax s_{\max}}{1+e^{-5\gammamax\smax}} .
\end{align*}
Applying Hoeffding's inequality yields
\begin{align*}
    &\Pr\big(\|\gradgam\cL(\bs,\bgamma)-\gradgam\bar\cL(\bs,\bgamma)\|_2\geq t\big)\leq 2m\exp\bigg(\frac{-(1+e^{-5\gammamax\smax})^2t^2mk}{100\gammamax^2\smax^2}\bigg),
\end{align*}
which immediately leads to the conclusion that
\begin{align*}
    \|\gradgam\cL(\bs,\bgamma)-\gradgam\bar\cL(\bs,\bgamma)\|_2\leq \frac{10\gammamax\smax}{1+e^{-5\gammamax\smax}}\sqrt{\frac{2\log(2mn)}{mk}}
\end{align*}
holds with probability at least $1-1/n$. This completes the proof.
\end{proof}

\subsection{Proof of Proposition \ref{prop:loss_convex_density_logconcave}}
\begin{proof}
Since the PDF $g$ of the noise terms $\epsilon_i$ is log-concave, and because the convolution of log-concave functions is log-concave~\cite{merkle1998}, the CDF $F$ of $\epsilon_j-\epsilon_i$ for any pair $i,j$ is also log-concave. Hence $h(x)=-\log F(x)$ is convex. The loss function is the sum of terms of the form $h_{iju}=h(\gamma_{u}(s_{i}-s_j))$. Fix $i$, $j$, and $u$. We have \[\hessb h_{iju}=h''(\gamma_{u}(s_{i}-s_j))(\gamma_{u})^2(\eb_i-\eb_j)(\eb_i-\eb_j)^\top,\]
where $\eb_i$ is the standard unit vector for coordinate $i$ in $\mathbb R^n$. By the convexity of $h$ and the fact that $(\eb_i-\eb_j)(\eb_i-\eb_j)^\top$ is positive-definite, the loss function is convex in $\bs$. Similarly, it is easy to show that it is convex in $\bgamma$. 
\end{proof}

\section*{Acknowledgment}
We would like to thank the anonymous reviewers for their helpful comments. We would like to thank Ashish Kumar for the collection of ``Country Population'' dataset. PX and QG are supported in part by the NSF grants CIF-1911168, III-1904183 and CAREER Award 1906169. TJ and FF are supported in part by the NSF grants CIF-1911168 and CCF-1908544. The views and conclusions contained in this paper are those of the authors and should not be interpreted as representing any funding agencies.

\bibliographystyle{ims}
\bibliography{LibraryMerged.bib}

\end{document}